%% file: main.tex
\documentclass{article}
\usepackage{iclr2027_conference,times}
\iclrfinalcopy
\usepackage{amsmath,amssymb,amsthm,mathtools}
\usepackage{booktabs,enumitem,microtype,flafter}
\usepackage[ruled,vlined]{algorithm2e}
\usepackage{hyperref,url}
\usepackage{graphicx,placeins,longtable,float}
\hypersetup{pdftitle={DataCupid: Stable Matching Repair for Training Data Acquisition with Delayed Feedback},
  pdfauthor={He Sun, Junyuan Mao, Wenhao Jiang, Mingjun Xiao, Shinan Liu}}
\newtheorem{theorem}{Theorem}
\newtheorem{lemma}{Lemma}
\newtheorem{proposition}{Proposition}
\newtheorem{corollary}{Corollary}
\newtheorem{assumption}{Assumption}
\newcommand{\E}{\mathbb E}
\newcommand{\Prb}{\mathbb P}
\newcommand{\ind}{\mathbf 1}
\newcommand{\cE}{\mathcal E}

\newcommand{\DA}{\operatorname{DA}}
\title{DataCupid: Stable Matching Repair for\\
Training Data Acquisition with\\
Delayed Feedback}
\author{%
He Sun\textsuperscript{1}\quad
Junyuan Mao\textsuperscript{2}\quad
Wenhao Jiang\textsuperscript{3}\quad
Mingjun Xiao\textsuperscript{4}\quad
Shinan Liu\textsuperscript{1}\\[0.6em]
\normalfont\textsuperscript{1}University of Hong Kong\\
\normalfont\textsuperscript{2}National University of Singapore\\
\normalfont\textsuperscript{3}Guangming Laboratory\\
\normalfont\textsuperscript{4}University of Science and Technology of China\\[0.4em]
\normalfont\small\parbox[t]{\dimexpr\textwidth-2\tabcolsep\relax}{\raggedright%
Email: \{\texttt{hesun@hku.hk} \& \texttt{he.sunny.sun@gmail.com}\}, \texttt{junyuanm@u.nus.edu},
\texttt{cswhjiang@gmail.com}, \texttt{xiaomj@ustc.edu.cn}, \texttt{shinan6@hku.hk}}
}
\begin{document}
\maketitle
\fancyhead{}
\begin{abstract}
Matching repair after utility changes can reuse valid observations, but new evaluations must respect matching capacity and may return out of order. We study announced invalidations with bounded, possibly reward-dependent reporting delays and pending evaluations that do not occupy future matching capacity. Residual-load repair (RLR) combines sound outcome certificates with batched matching schedules for unequal sample deficits. Its complete dyadic quota schedule admits a sharp issue-round bound with coefficient $c_b<1.607$ relative to the terminal residual endpoint load, together with separate delay and evidence-lifetime bounds. RLR-P monitors complete issue-order prefixes without changing the sampling plan before acceptance. With a deterministic pure oracle, a single-repair coupling from the same drained initial state guarantees no increase in certification time, issued observations, or time to a reusable checkpoint. We also derive an information-load lower bound. On fixed priority-star markets with unknown Bernoulli gaps, Priority-Search attains the history-adjusted first-order constant for expected certification time as $\delta\downarrow0$, under hypothesis-independent linear history scaling and sublogarithmic reporting delays. Paired simulations and measured-feedback replay quantify the trade-offs: in the replay, RLR-P has lower mean capped update-certification time than a monitored width-based baseline, at higher oracle CPU cost.
\end{abstract}
\section{Introduction}

A matching platform must learn uncertain utilities while respecting the capacity constraints that govern its allocations. Provider priorities couple requesters' choices, so independently identifying each requester's favorite provider need not yield a stable matching \citep{liu2020competing,basu2021decentralized}. With known preferences, requester-proposing deferred acceptance (DA) computes the requester-optimal stable matching \citep{gale1962college}. After a utility update, however, rerunning DA does not resolve uncertainty about the new preferences. We study repair after announced invalidations, where observations outside the notified set remain valid. The goal is to use this retained evidence and new evaluations to certify the updated matching.


\paragraph{Delayed evaluation in a data market.}
One motivating setting is a data marketplace, where buyers compete for access to sellers' datasets \citep{agarwal2019marketplace}. Related crowdsensing markets require learning unknown buyer preferences while maintaining stable trading relationships \citep{sun2023privacy,sun2024crowdsensing}. Evaluating an offer can require training or fine-tuning a model from a fixed checkpoint and measuring held-out performance \citep{ghorbani2019datashapley}, so feedback becomes available only after evaluation completes. A dataset revision can invalidate evaluations of its previous version while leaving those of unchanged offers reusable. We model capacity-limited access grants followed by asynchronous evaluations that do not occupy later access slots. Buyer tasks and starting checkpoints remain fixed while their observations are retained.


\paragraph{Why the location of missing evidence matters.}
Suppose an update leaves $b$ observations to collect on each of $N$ pairs. With disjoint endpoints, all $Nb$ observations can be issued in $b$ rounds; if the pairs share one requester, they require at least $Nb$ rounds. The sample total is identical, but the feasible parallelism differs. This motivates the \emph{residual endpoint load}: the largest total sample deficit incident to any requester or provider. Delays introduce a further challenge. Waiting after every matching can repeatedly incur reporting delay, even for evaluations planned together. Yet using only the results that arrive first can bias the evidence when delays depend on rewards. Efficient repair therefore requires both a feasible sampling schedule and a valid way to recognize when sufficient evidence has returned.


We use \emph{outcome certificates} to make this stopping decision. A sound certificate confirms that every utility matrix within the current confidence bounds yields the same requester-optimal matching. Existing partial-preference, PCOS, and pervasive-matching methods can certify outcomes without resolving every preference comparison \citep{rastegari2014reasoning,athanasopoulos2025pcos,athanasopoulos2026probably}. Building on this interface, we ask: \emph{how do retained observations, matching capacity, and delayed reporting jointly determine the time needed to certify an updated allocation?} The certificate specifies when the evidence is sufficient; the repair protocol organizes how it is collected and incorporated.

First, we develop \emph{residual-load repair} (RLR), which packs unequal sample deficits into feasible matchings using doubling count targets and a single feedback barrier per phase. For a complete dyadic quota schedule ending at target $b$, we prove a sharp issue-round bound of $c_b$ times its terminal residual endpoint load, where $c_b<1.607$; reporting delay contributes a separate term for each nonempty phase. An evidence-lifetime analysis bounds cumulative new samples across updates, crediting observations for as long as they remain valid.


Second, we introduce \emph{prefix-monitored RLR} (RLR-P), which checks the certificate whenever a complete prefix of issued observations becomes available. Such prefixes support valid confidence bounds even when reporting delays depend on rewards. Before certification, the monitor leaves RLR's sampling plan unchanged. With a deterministic, pure oracle, we couple both methods for a single repair from the same initial state with no pending evaluations. RLR-P certifies no later, issues no more observations, and finishes receiving all issued feedback no later than RLR. It serves the certified matching while completing outstanding evaluations and retaining their results for later updates.


Third, we connect repair time to necessary statistical information. A change-of-measure lower bound credits historical evidence and imposes the same endpoint capacities on new observations. We establish achievability on fixed \emph{priority-star} markets, where each requester has a known private fallback and competes for a shared provider within its group. For unknown Bernoulli utilities with nonzero gaps to the fallback, \emph{Priority-Search} tests candidates in public priority order. For hypothesis-independent history sizes scaling linearly with $\log(1/\delta)$ and sublogarithmic reporting delays, it attains the history-adjusted information-load constant for expected certification time as $\delta\downarrow0$. Paired simulations and measured-feedback replay examine how these mechanisms trade earlier certification against additional oracle computation.

\section{Related work}
Stable matching and competing bandits provide the allocation objective and established exploration schemes \citep{gale1962college,roth1992two,liu2020competing,basu2021decentralized,kong2024many,kong2025indifference}. Pervasive inference and partial-preference methods identify stable outcomes without complete rankings \citep{rastegari2014reasoning,aziz2020uncertain,athanasopoulos2025pcos,athanasopoulos2026probably,basu2025superstable}; we use such logic as stopping components. Bipartite edge coloring is classical and also appears in PCOS. Our scheduling result concerns residual multiplicities, delayed phase completion, and their reuse across evidence lifetimes.

Crowdsensing data trading provides a closely related application of matching bandits. \citet{sun2023privacy} combine differential privacy, stable matching, and competing bandits to learn unknown buyer preferences while protecting sellers' data qualities. Their journal extension adds a decentralized mechanism to mitigate competing matching requests \citep{sun2024crowdsensing}. These works address privacy, trading stability, and matching conflicts. The present work studies how retained evidence and matching capacity determine repair time after announced invalidations with delayed feedback.

Combinatorial pure exploration optimizes statistical sampling allocations \citep{chen2017nearly,jourdan2021combinatorial}, while batch complexity measures how adaptivity rounds interact with sample efficiency \citep{tuynman2025batch}. Offline-data best-arm identification explicitly credits historical information \citep{agrawal2023offline}, and biased histories require additional safeguards \citep{yang2025biased}. We assume announced invalidations and exact mean preservation on retained edges. The new element is the matching-capacity load imposed on residual information. Priority-Search builds on threshold identification and ordered good-arm search \citep{locatelli2016thresholding,kano2019good}; cascading and satisficing bandits use related ordered stopping ideas under different feedback and objectives \citep{kveton2015cascading,russo2022satisficing}.

Delayed learning and asynchronous thresholding have established guarantees \citep{joulani2013delayed,lancewicki2021unrestricted,zhong2017asynchronous}. RLR-P differs operationally because return order may depend on reward: only complete issue-order prefixes enter confidence intervals, and the canonical planning state changes only at the original barrier. Nonstationary matching and computational preference repair study related changes \citep{muthirayan2023competing,ghosh2022decentralized,bredereck2019adapting,gangam2026robust}. We focus on repair after exogenous change notifications, with sufficient time between updates for each nonfinal repair to complete.

\section{Problem formulation}\label{sec:model}
We consider $N$ requesters and $K$ providers, where $1\le N\le K$. Each participant can be matched to at most one counterpart per round.
The feasible graph $E$ and each provider's strict priority order
$\succ_j$ are fixed and public. Every edge is acceptable to its provider; requester acceptability is unknown. Write $F_i=\{j:(i,j)\in\cE\}$, $M(i)=0$ for an unmatched requester, and $M^{-1}(j)=0$ for an unmatched provider. The outside option has utility zero. If $\cE$ is empty, all losses are zero.

\begin{assumption}[Exogenous rewards]\label{ass:reward}
The mean sequence $\mu_{i,j,t}\in[-1,1]$ is fixed before learning. An accepted edge $e=(i,j)$ generates $Y_{e,t}\in[-1,1]$ with centered noise conditionally $\sigma$-sub-Gaussian given the full latent history before round $t$. Dependence across edges within a round is permitted. Set $\mu_{i,0,t}=0$.
\end{assumption}
\begin{assumption}[Reporting delay]\label{ass:delay}
Feedback generated at $s$ arrives at the end of $s+d_{e,s}$, with integer $0\leq d_{e,s}\leq d_e<\infty$. Decisions at $t$ use only arrivals with $s+d_{e,s}<t$. Delays may depend on rewards. Acceptance is immediate and supplies no additional reward information; pending evaluation does not occupy future matching capacity. Write $d=\max_e d_e$ for the baseline analysis.
\end{assumption}
\begin{assumption}[Cooperative protocol]\label{ass:protocol}
Public IDs, priorities, graph, and a common clock are available. Proposal messages and interval checks take place between reward-generating rounds. Participants cooperate with scheduled exploration, including temporarily suspending an incumbent so that one designated edge can be sampled. No strategic-compliance or communication-free guarantee is claimed.
\end{assumption}
\begin{assumption}[Announced invalidations]\label{ass:announced}
Means are constant in epochs whose boundaries are fixed before learning. At the start of epoch $h$, a public set $\mathcal C_h\subseteq\cE$ contains every edge whose mean changed; false-positive invalidations are allowed and valid observations outside $\mathcal C_h$ are retained. An evidence lifetime is the interval between notifications involving the same edge. The multi-update guarantee assumes each nonfinal repair is certified and drained before the next notification; a sufficient dwell condition is stated with the repair-time bound. An earlier interruption follows the declared fallback and is not a completed-repair guarantee. 
\end{assumption}
Let $M_h^\star=\DA(\mu_h)$ be requester-proposing DA with a fixed provider-ID tie rule and the outside option preferred at utility zero. This is the requester-optimal outcome for the chosen refinement. 
\paragraph{Performance measures.}
We measure service quality during repair through positive-part
dynamic stable regret and stability violations.
\begin{equation}
 \mathcal R_T=\sum_{t=1}^T\sum_i[\mu_{i,M_t^\star(i),t}-\mu_{i,M_t(i),t}]_+,
 \qquad R_T=\E\mathcal R_T .\label{eq:regret}
\end{equation}
For any feasible $M$, let $B_\varepsilon(M,t)$ contain pairs $(i,j)\in\cE$ with $j\ne M(i)$, $i\succ_j M^{-1}(j)$, and $\mu_{i,j,t}>\mu_{i,M(i),t}+\varepsilon$. Let $I_\varepsilon(M,t)=\{i:\mu_{i,M(i),t}<-\varepsilon\}$ and $U_\varepsilon=|B_\varepsilon|+|I_\varepsilon|$. Then $S_T^\varepsilon=\E\sum_tU_\varepsilon(M_t,t)$ measures stability violations, including individual rationality. Provider rationality is built into the graph. Always $R_T\leq2NT$ and $S_T^\varepsilon\leq N(K+1)T$. Appendix~\ref{sec:certificate} recalls the distinction between small deviation incentives and optimal stable regret.

\section{Certificates and valid evidence}\label{sec:certificates}
For a nonempty interval box $\mathcal B=\prod_{e\in E}[\ell_e,u_e]\subseteq[-1,1]^E$, a \emph{sound} oracle returns $\mathrm{unresolved}$ or a matching equal to $\mathrm{DA}(\nu)$ for every $\nu\in\mathcal B$, under the fixed graph, priorities, and tie refinement. It supplies a stopping decision; the repair protocol selects evaluations. A \emph{resolution radius} $\rho_h>0$ guarantees acceptance of every box containing $\mu_h$ with maximum width at most $\rho_h$. It is needed for the repair-time bound and is not a policy input.

Forced-P combines a canonical trace check with mutual-first-choice contraction. Hybrid adds exhaustive possible-weak-stability testing: a unique surviving matching certifies the refined DA outcome. When the true weakly stable matching is unique, its exclusion margin gives $\rho_h=\beta/3$ (Theorem~\ref{thm:exclusion}). Enumeration is exponential and restricted to small markets. PSM-P uses the published pervasive oracle on certified-positive complete markets. Definitions and proofs appear in Appendices~\ref{sec:outcome},~\ref{app:modular},~\ref{app:exclusion}, and ~\ref{sec:trace}.

Within each validity lifetime, intervals start at $[-1,1]$ or a valid prior and use complete issue-order prefixes (\ref{sec:load}). The empirical Bernstein alternative additionally assumes an unchanged i.i.d. bounded stream within each lifetime (Eq.~\eqref{eq:eb-main}).

\input{sections/load_theory}

\paragraph{General information lower bound.}
For independent Bernoulli edge rewards and different-outcome alternatives $\lambda$, let $D_0(\lambda)$ be the KL divergence of the initial information. Compared markets must have a common conditional reporting channel given issued rewards and history; historical pending feedback must be absent or have an uninformative common conditional channel as specified in Theorem~\ref{thm:warm-lower}. Then necessary expected new counts $x_e$ satisfy
\[
 \sum_e x_e\operatorname{kl}(\mu_e,\lambda_e)
 \geq[\operatorname{kl}(1-\delta,\delta)-D_0(\lambda)]_+.
\]
Minimizing $\mathsf L(x)$ over these constraints gives Eq.~\eqref{eq:history-load}, a necessary elapsed-time bound. The priority-star theorem makes the resulting constant achievable for its fixed structured class with unknown gaps. For general markets, choosing a statistically efficient demand vector, rather than realizing a given vector, remains open. Delay may add costs beyond this information-load lower bound.

\input{sections/load_experiments_main}
\section{Limitations}\label{sec:limitations}
RLR meets prescribed precision but may oversample outcome-irrelevant edges and is not generally statistically optimal. T-Monitor assumes a deterministic pure oracle and covers one repair, giving no multi-update or CPU-time dominance. T-Star requires a fixed priority-star graph, independent Bernoulli streams, hypothesis-independent histories, a common report channel, and sublogarithmic delay. Announced invalidations, nonoccupying jobs, and sufficient dwell delimit the repair theory. The measured replay is a finite empirical distribution, and the supported PSM-P stopping oracle is not full E-PSM.

\section{Conclusion}
Endpoint congestion and report barriers determine how residual evidence becomes repair time. RLR has the sharp $c_b<1.607$ execution coefficient for its complete dyadic schedule; RLR-P safely acts on complete prefixes and exports checkpoints after drainage. Priority-Search attains the information constant on a structured unknown-gap class. General matching instance optimality remains open.

\bibliographystyle{iclr2027_conference}
\bibliography{iclr2027_conference}
\clearpage
\appendix
\input{sections/load_appendix}

\input{sections/load_reference}
\input{sections/outcome_appendix}
\input{sections/exclusion_appendix}
\input{sections/load_experiment_details}
\input{sections/pre_load_experiments}
\input{sections/trace_reference}

\input{sections/trace_proofs}
\input{sections/stability_reference}

\input{sections/new_experiment_details}
\end{document}

%% file: sections/load_theory.tex
\section{Prefix-monitored residual repair}\label{sec:load}
\paragraph{Residual demands and feasible execution.} 
For a nonnegative edge-demand vector $x$, define its \emph{endpoint load} by
\begin{equation}
 \mathsf L(x)=\max\left\{\max_i\sum_jx_{ij},\;\max_j\sum_ix_{ij}\right\}.
 \label{eq:residual-load}
\end{equation}
Thus $b$ missing observations on each of $N$ disjoint edges cost $b$ issue rounds, while the same deficits on $N$ edges incident to one requester cost $Nb$ rounds.

\begin{lemma}[Exact quota realization]\label{lem:quota}
For integer total-count targets $z_e$, let $x_e=[z_e-m_e]_+$, where $m_e$ is the valid completed count. Exactly $\mathsf L(x)$ rounds are necessary and sufficient to issue these observations as feasible matchings. 
If every issued observation reports within $d$ rounds of its issuance, all prescribed feedback is available within at most $d$ further rounds after the final issue.
\end{lemma}
This is the bipartite multigraph edge-coloring identity. Our implementation finds a compressed decomposition without materializing each parallel edge; the coloring fact itself is classical and is also used for unit sweeps in PCOS \citep{athanasopoulos2025pcos}.

\paragraph{Evidence records and repair clocks.}
Within lifetime $g$ of edge $e$, $n^{\mathrm{iss}}_{e,g}$ counts all learning observations designated before their rewards or delays are seen; $n^{\mathrm{arr}}_{e,g}$ counts returned observations; and $n^{\mathrm{pre}}_{e,g}$ is the largest fully returned prefix in issue order. The canonical count $n^{\mathrm{can}}_{e,g}$ is committed at an RLR phase barrier. Thus $n^{\mathrm{can}}_{e,g} \le n^{\mathrm{pre}}_{e,g} \le n^{\mathrm{arr}}_{e,g} \le n^{\mathrm{iss}}_{e,g}$. For example, returns $2,3,4$ remain buffered while return $1$ is missing. All issued observations are charged, including those still pending. We write $n^0_{e,h}$ for the valid count inherited by repair $h$, and reserve $n^0_{e,g}$ for the initial count credited once in the lifetime-level accounting.

The certification time $\tau^{\mathrm{cert}}$ is the number of reward rounds until the policy first receives an accepting oracle answer; it is zero when an initial check accepts. Readiness $\tau^{\mathrm{ready}}$ additionally requires all observations issued by that repair to have returned and their valid prefixes to have been committed. Both times lie in $\{0,1,2,\ldots\}\cup\{\infty\}$. Under an epoch budget $H_h$, we use $\bar\tau_h^{\mathrm{cert}}=\min\{\tau_h^{\mathrm{cert}},H_h\}$ and the analogous readiness cap, together with separate completion indicators. The certified matching is served while any outstanding evaluations finish.

Fix a confidence-record horizon $T_{\mathrm{conf}}$ covering the lifetime indices and within-lifetime prefixes used by the policy. Set
\[
\Lambda_{\mathrm{conf}}
=
\log\frac{2\max\{1,|E|\}T_{\mathrm{conf}}^2}{\delta},
\qquad
r(n)=\sigma\sqrt{\frac{2\Lambda_{\mathrm{conf}}}{n}},
\quad n\ge1.
\]
If $\bar Y_{e,g,n}$ is the mean of the first $n$ issued learning observations in lifetime $(e,g)$, the corresponding raw interval is $[\bar Y_{e,g,n}-r(n),\bar Y_{e,g,n}+r(n)]\cap[-1,1]$. Each evidence view intersects the valid prior and the intervals committed to its own archive within that lifetime.
\begin{theorem}[Complete-prefix validity (T-Prefix)]\label{thm:prefix}
Suppose every lifetime index and within-lifetime prefix used in the analysis is at most $T_{\rm conf}$. If designation as a learning observation is predictable before its reward and delay are observed, then, with probability at least $1-\delta$, every interval formed from every issue-ordered complete prefix contains its current mean. This holds for adaptive sampling and reward-dependent bounded delays, and remains true after intersecting intervals within a lifetime.
\end{theorem}
The simultaneous event covers every prefix length, so it also covers the random length available at a reporting time. We denote this event by $\mathcal G$. Appendix~\ref{app:load} proves the result and gives an anytime alternative; the empirical Bernstein extension uses the additional i.i.d. assumption stated in Section~\ref{sec:certificates}.



\begin{algorithm}[t]
\caption{Prefix-monitored residual-load repair (RLR-P)}\label{alg:rlr-p}
\KwIn{drained valid evidence, incumbent $M^{\rm old}$, invalidations, pure oracle}
Reset only invalidated lifetimes; retain valid canonical and monitor views\;
\For{$z=1,2,4,\ldots$ while the canonical oracle is unresolved}{
 Plan once from the canonical view: $x_e\leftarrow[z-n_e^{\rm can}]_+$ and a fixed exact decomposition\;
 \For{each planned matching block $A_t$}{
  Serve $\operatorname{ins}(M^{\rm old},A_t)$ and register every issue\;
  Cache arrivals; after \emph{each} complete-prefix advance, query the monitor box\;
 \If{certified}{cancel unissued blocks; serve the certified matching; \textbf{break}}
 }
 Drain issued jobs; commit complete prefixes at the canonical barrier\;
 \If{the monitor certified}{return its matching and the drained checkpoint}
 Query the oracle on the original canonical box; stop if it certifies\;
}
\end{algorithm}


\paragraph{The RLR sampling plan.}
RLR starts each repair by invalidating only the notified edge lifetimes. It then uses total-count targets $z=1,2,4,\ldots$. At each target, the canonical record determines the deficits $x_e=[z-n_e^{\mathrm{can}}]_+$. Lemma~1 gives an exact matching decomposition, whose order is fixed before the phase starts. RLR issues the planned evaluations, waits for the remaining reports, commits the completed phase to the canonical archive, and checks the outcome certificate.

\paragraph{Monitoring complete prefixes.}
RLR-P keeps this plan and adds a separate monitor archive. Whenever an issue-order prefix becomes complete, it updates this archive and queries the same deterministic, stateless oracle. Raw observations may be shared, but the interval-intersection archives remain separate. Before acceptance, monitor queries do not modify the canonical targets, block order, or barrier updates. At every original phase barrier, RLR-P retains the oracle check on the canonical box. The monitor continues to run while an uncertified phase is waiting for its remaining reports.

During exploration, the service allocation is $\mathrm{ins}(M^{\mathrm{old}},A_t)$: remove incumbent pairs incident to the exploratory matching $A_t$, then insert $A_t$. A compressed block is executed one reward round at a time, with all issues in a round registered before its returns are processed. At the first acceptance, RLR-P freezes the output, cancels only future unissued rounds, and serves the certified matching. Already issued evaluations finish, and their valid prefixes are committed before the checkpoint is reused. Algorithm~\ref{alg:rlr-p} summarizes this execution; Appendix~\ref{app:a3} gives its event-level conditions.

\begin{theorem}[Single-repair monitoring and causal coupling (T-Monitor)]\label{thm:monitor}
Start RLR and RLR-P from the same drained evidence, incumbent, canonical archive, and scheduler state. Couple their environments by a common conditional kernel whenever latent histories and current actions agree; report dates are fixed at issue. Then, pathwise for this repair,
\begin{equation}
 \tau^{\rm cert}_{\rm P}\le\tau^{\rm cert}_{\rm RLR},\qquad
 P^{\rm iss}_{\rm P}\le P^{\rm iss}_{\rm RLR},
 \label{eq:monitor-dominance}
\end{equation}
and RLR-P's issue sequence is a prefix of RLR's. The inequalities also hold for the capped quantities. With report times fixed at issue, draining gives $\tau^{\rm ready}_{\rm P}\le\tau^{\rm ready}_{\rm RLR}$. On the T-Prefix event every RLR-P output is correct.
\end{theorem}
The comparison holds for one repair from common initial evidence. Later repairs may begin with different retained histories, and the additional oracle calls are measured separately from reward-round costs in Section~\ref{sec:loadexp}.

\paragraph{Load, waiting, and lifetime accounting.}
For $b=2^J$, write $c_b=b^{-1}+\sum_{k=1}^J(2^k-1)^{-1}$ and
$c_\star=\sum_{k\ge1}(2^k-1)^{-1}<1.607$. Let
$\Delta_E=\mathsf L(\mathbf1_E)$ denote the maximum feasible-graph degree.
Let $\rho_h>0$ resolve the oracle in epoch $h$ and define
\begin{equation}
 q_h=\max\{1,\lceil8\sigma^2\Lambda_{\rm conf}/\rho_h^2\rceil\},\quad
 b_h=2^{\lceil\log_2q_h\rceil},\quad B_h=1+\log_2b_h,
 \qquad D_{e,h}=[b_h-n^0_{e,h}]_+ .
 \label{eq:load-threshold}
\end{equation}
The bound applies only while the epoch remains unchanged and the confidence
record envelope covers every lifetime and prefix through $b_h$; an external
notification or horizon interruption yields a capped/fallback record rather
than a completed-repair guarantee.
\begin{theorem}[Residual-load and lifetime bounds]\label{thm:load}
On the T-Prefix event, an RLR-P repair starting drained certifies and becomes ready by
\begin{equation}
 \tau_h^{\rm cert}\le\tau_h^{\rm ready}
 \le \min\{c_{b_h}\mathsf L(D_h),\,b_h\Delta_E\}+dB_h. \label{eq:residual-rounds}
\end{equation}
For a nonempty phase with $Q_s$ actual issue rounds, let observation $(e,k)$ be issued in relative round $u_{e,k}$ with realized delay $D_{e,k}$, let $l_e=\max_k u_{e,k}$, and let $A_s$ contain the edges actually issued. Its exact post-issue wait and useful bounds are
\begin{equation}
 W_s=\left[\max_{e,k}(u_{e,k}+D_{e,k})-Q_s\right]_+
 \le\max_{e\in A_s}[d_e-(Q_s-l_e)]_+
 \le\max_{e\in A_s}d_e. \label{eq:fine-wait}
\end{equation}
If each nonfinal repair drains before the next notification and valid observations are retained, let $b^{\max}_{e,g}$ be the largest $b_h$ over repairs intersecting lifetime $g$ of edge $e$. With $n^0_{e,g}$ denoting the valid count initially credited to that lifetime, the new observations satisfy
\begin{equation}
 P_{e,g}^{\rm new}\le[b^{\max}_{e,g}-n^0_{e,g}]_+,
 \qquad P^{\rm new}\le\sum_{e,g}[b^{\max}_{e,g}-n^0_{e,g}]_+,
 \label{eq:load-lifetime}
\end{equation}
\end{theorem}
The first term in Eq.~\ref{eq:residual-rounds} exploits the inherited evidence, whereas $b_h\Delta_E$ gives a history-independent issuing bound. Both share the phase-delay term $dB_h$. The coefficient $c_b$ is sharp for complete execution of the prescribed dyadic plan; a successful certificate may stop that plan earlier. Equation~\ref{eq:fine-wait} charges waiting only to issued edges and accounts for delay hidden by later issues. Equation~\ref{eq:load-lifetime} credits valid observations across updates, bounding new sampling by the largest target reached during an evidence lifetime. The same bounds apply to RLR, and $b_h\Delta_E+dB_h$ is a sufficient history-independent dwell condition.

\paragraph{Service cost during repair.}
Let $P_h$ count issued exploratory observations and let $a_h$ count requesters whose incumbent assignment differs from the current target. On $\mathcal G$, service switches to the target at certification, including during the final drain. Inserting exploration changes only assignments incident to the exploratory matching. The resulting pathwise bounds are
\begin{align}
 \mathcal R_T&\le\sum_h\{2a_h\bar\tau_h^{\rm cert}+4P_h\},\label{eq:load-regret-cert}\\
 \mathcal S_T^0&\le(K+2N+1)\sum_h\{a_h\bar\tau_h^{\rm cert}+2P_h\}.\label{eq:load-stability-cert}
\end{align}
Taking expectations adds $2NT\delta$ and $N(K+1)T\delta$, respectively. These are good-event path inequalities followed by failure charges; 
Taking expectations and applying the universal loss bounds on $\mathcal G^c$ adds at most $2NT\delta$ and $N(K+1)T\delta$, respectively. Appendix~\ref{Sec:A1-5} gives the insertion and failure-event arguments.

\subsection{A structured unknown-gap characterization}\label{sec:priority-star}
The preceding results schedule prescribed quotas but do not choose an information-optimal vector. We can close that gap on a fixed finite \emph{priority-star} class. There are disjoint requester groups $I_j$, one unit-capacity shared provider $j$ per group, and for each requester $i\in I_j$ a private provider offering known utility $\theta\in(0,1)$. Candidate edge $(i,j)$ has unknown i.i.d. Bernoulli mean $p_i\ne\theta$; each shared provider has a public strict priority order. Multiple candidates may exceed $\theta$.

Let $a_j$ be the first priority position with $p_i>\theta$, or $\infty$ if none, and let $\mathcal B_j$ contain the prefix through $a_j$, or all of $I_j$ if $a_j=\infty$. The unique stable allocation assigns $j$ to $a_j$ when it exists and assigns every other requester its private fallback. Thus every $i\in\mathcal B_j$ is outcome-pivotal under a one-coordinate alternative across $\theta$.

Suppose the initial data are independent current-distribution prefixes of deterministic, hypothesis-independent sizes $m_i(\delta)$, with no pending reports and $m_i(\delta)/\ell\to c_i$, where $\ell=\log(1/\delta)$. More precisely, the candidate reward streams are independent across edges and issue indices (and independent of the deterministic history); hypotheses share the same conditional reporting channel given issued rewards and the past, pending work does not consume later access capacity, and $d(\delta)=o(\ell)$. Put $r_i=(\operatorname{kl}(p_i,\theta)^{-1}-c_i)_+\ind\{i\in\mathcal B_j\}$ on candidate edge $(i,j)$.

\begin{theorem}[Priority-star information load (T-Star)]\label{thm:priority-star}
For every fixed priority-star graph and every fixed vector $p$ with $p_i\ne\theta$, any procedure uniformly $\delta$-correct over the whole class satisfies
\begin{equation}
 \liminf_{\delta\downarrow0}\frac{\E_p\tau_\delta}{\ell}
 \ge L_\star(p,c):=\mathsf L(r)=\max_j\sum_{i\in\mathcal B_j}
 \left(\operatorname{kl}(p_i,\theta)^{-1}-c_i\right)_+ .
 \label{eq:priority-star-limit}
\end{equation}
Here $\tau_\delta=\tau_\delta^{\rm cert}$ counts new reward rounds. The mean-blind Priority-Search procedure described below is uniformly $\delta$-correct and, on every such fixed instance,
$\lim_{\delta\downarrow0}\E_p\tau_\delta/\ell=L_\star(p,c)$.
If $L_\star=0$, this means $\E_p\tau_\delta=o(\ell)$.
\end{theorem}
\begin{corollary}[Priority-star readiness]
Under the assumptions of T-Star, stopping new evaluations after certification
and draining all issued jobs gives the same first-order limit
$L_\star(p,c)$ because $d(\delta)=o(\log(1/\delta))$.
\end{corollary}
Priority-Search tests candidates in public priority order, stopping a group at its first positive decision and running different groups in parallel. For $M=\sum_j|I_j|$, it scans every complete prefix and freezes a sign at the first $n$ satisfying
\begin{equation}
 n\operatorname{kl}(\widehat p_{i,n},\theta)
 \ge \log\frac{2Mn^2}{\delta}. \label{eq:priority-glr}
\end{equation}
It uses neither gaps nor the unknown relevant prefixes. The statistical ingredients are related to thresholding and good-arm identification \citep{locatelli2016thresholding,kano2019good}; ordered stopping also appears under different bandit objectives \citep{kveton2015cascading,russo2022satisficing}. The contribution here is the resulting history-adjusted matching-capacity constant on this restricted class. It is separate from RLR-P and gives no general matching instance-optimality claim. The full proof, including the error-path expectation and delay accounting, is in Appendix~\ref{app:priority-star}. The earlier public-gap, single-positive fallback benchmark is retained in Appendix~\ref{app:known-gap}.

%% file: sections/load_experiments_main.tex
\section{Experiments}\label{sec:loadexp}
We evaluate the claims at the level at which they are made. The scheduling diagnostics use prescribed residual quotas; stochastic comparisons use paired reward streams and retain every capped trajectory. RLR-P and Width-P share the same prefix-monitor wrapper as their canonical counterparts. All four policies use the same confidence level ($\delta=.01$), horizon, market instance, and evaluator oracle within a cell. Certification time, readiness, issued evaluations, regret, and oracle CPU are reported separately. The implementation uses an exact complete-input cache for the pure oracle: logical queries are retained, and only byte-equivalent full boxes with identical immutable configuration may reuse an answer. Cache hits cannot alter the issue or service sequence; they are a computation accounting optimization rather than a statistical or wall-clock guarantee.
\vspace{-0.2in}
\begin{figure}[!htbp]
\centering
\includegraphics[width=0.88\linewidth]{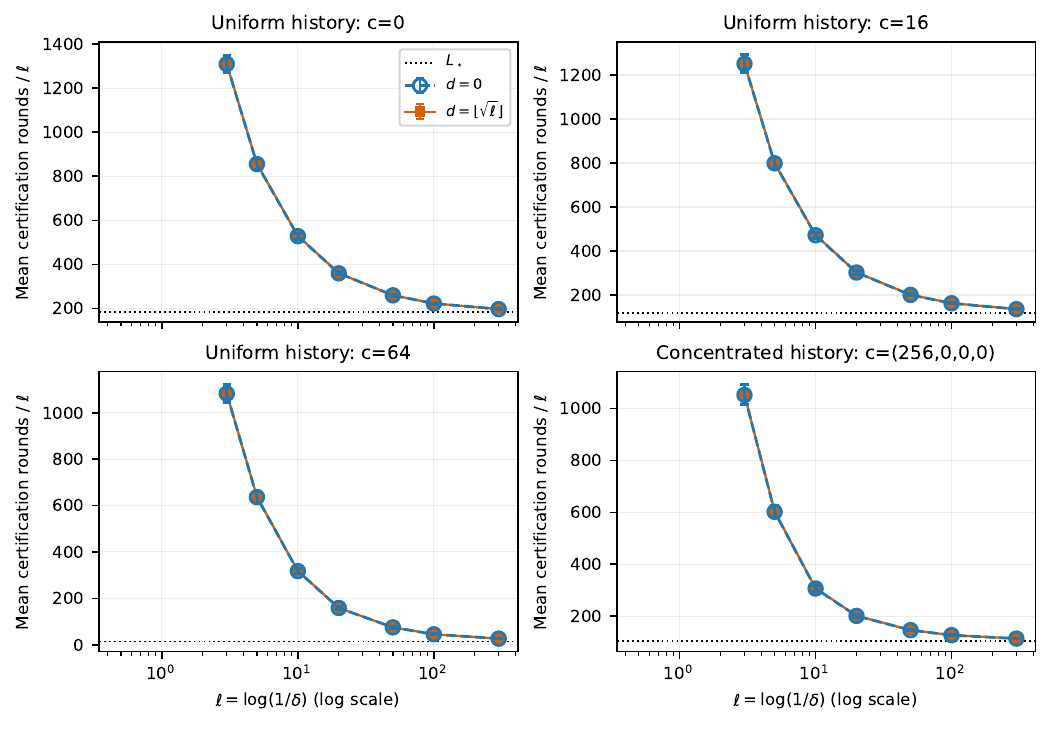}
\caption{Unknown-gap Priority-Search on the fixed priority-star class. The horizontal axis is $\ell=\log(1/\delta)$ on a logarithmic scale; the ordinate is mean new certification rounds divided by $\ell$. Dashed lines show the asymptotic information-load constant $L_\star(p,c)$, not a finite-confidence lower bound. Blue open circles with dashed lines use $d=0$; orange squares with solid lines use the audited $d=\lfloor\sqrt{\ell}\rfloor$ regime. Error bars are percentile 95\% intervals from 10,000 bootstrap resamples of the 100 per-configuration seeds. The curves are numerically distinct but close at the displayed scale; paired differences are audited in Appendix~\ref{app:priority-star}.}
\label{fig:priority-star}
\vspace{-0.2in}
\end{figure}

\paragraph{Prescribed-load and prefix checks.}
The frozen-quota suite contains 900 instances over five market sizes, six demand families, and 30 seeds per cell. Exact decomposition reaches the endpoint-load lower bound on every instance and satisfies the phase wait identity. A greedy decomposition ties the exact issue time on 895/900 instances; its largest observed ratio is 1.0922. Thus the theorem's endpoint-load statement is an exact resource identity, while a broad CPU or regret advantage for exact packing is not supported. These are evaluator-controlled scheduling tests, not identification trials.

\paragraph{Unknown-gap structured characterization.}
This new statistical experiment is separate from the earlier public-gap likelihood benchmark. On the fixed priority-star class, candidate gaps are unknown to the policy, several candidates may exceed the known fallback threshold, and historical count vectors are fixed across alternatives. The released run contains 5,600 trials over the prescribed history, delay, and distribution conditions, with no observed wrong output. At $\ell=300$ and zero history, normalized certification time is 197.487 versus the first-order constant 183.719, revealing remaining finite-confidence overhead. With history calibrated to remove the first-order load, empirical time remains positive due to testing and finite-sample overhead. These observations illustrate the pointwise asymptotic theorem but do not establish nominal error probabilities or general matching instance optimality.

\input{results/load/e5_merged_table}

\paragraph{Native certificate comparison with monitoring.}
The completed E5 suite covers 120 paired markets and 1,680 branches over $N\in\{3,5,8,16\}$, 30 seeds, and budgets $30{,}000$ and $100{,}000$, with no implementation failures and every capped path retained. Table~\ref{tab:e5-merged} reports raw completion counts and capped means. For $N=3$, RLR-P, Width-P, and PCOS-A4 complete respectively 18, 16, and 13 paths at the smaller budget, and 23, 22, and 20 at the larger budget. These finite-budget results do not establish universal dominance or computational speedup.
\vspace{-0.1in}
\input{results/load/final_quality_table}
\input{results/load/final_quality_macros}

\paragraph{Independent histories and measured feedback.}
The frozen ten-split replay is complete (30 tasks, 960 epoch rows, 480 trajectories, no implementation failures). Table~\ref{tab:r5-quality} retains all 30 selected paired paths, including censored updates. After averaging policy seeds within split, RLR-P minus Width-P has update-certification difference $\TableTwoUpdateCertDiff$ rounds (95\% CI $[\TableTwoUpdateCertDiffLo,\TableTwoUpdateCertDiffHi]$), update-regret difference $+\TableTwoUpdateRegretDiff$ (CI $[\TableTwoUpdateRegretDiffLo,\TableTwoUpdateRegretDiffHi]$), and cumulative oracle-CPU difference $+\TableTwoTotalCpuDiff$ seconds (CI $[\TableTwoTotalCpuDiffLo,\TableTwoTotalCpuDiffHi]$). These finite-budget comparisons do not imply deployment speedup or uncensored stopping-time dominance.

\paragraph{Interpretation.}
Exact packing realizes the geometric lower bound, monitoring can reduce one repair's certification cost under the fixed planner, and Priority-Search matches its constant only on the priority-star family. The phase excess $\Omega=\min_v\sum_s(q_s-a_{v,s})$ is a bottleneck change; Appendix~\ref{app:load} separates it from the shadow suffix and drain. The evidence establishes neither RLR-P dominance nor general matching instance optimality.

%% file: results/load/e5_merged_table.tex
\begin{table}[H]
\centering\scriptsize
\caption{E5 same-instance finite-budget comparison. Each cell gives correct certificates/30 and mean capped rounds. All cells retain censored paths; PCOS-A4 is native and PSM-P is separate.}
\label{tab:e5-merged}
\begin{tabular}{llrrrrr}
\toprule
$N$&Budget&RLR-P&Width-P&PCOS-A4&RLR&Width\\\midrule
3&30k&18/30; 18095&16/30; 20005&13/30; 24107&17/30; 19374&13/30; 20935\\
3&100k&23/30; 38130&22/30; 41841&20/30; 55573&23/30; 42815&22/30; 44280\\
5&30k&6/30; 27152&4/30; 28585&1/30; 29917&4/30; 27877&3/30; 28705\\
5&100k&7/30; 82101&7/30; 86256&2/30; 95918&7/30; 84005&7/30; 87581\\
8&30k&0/30; 30000&0/30; 30000&0/30; 30000&0/30; 30000&0/30; 30000\\
8&100k&0/30; 100000&0/30; 100000&0/30; 100000&0/30; 100000&0/30; 100000\\
16&30k&0/30; 30000&0/30; 30000&0/30; 30000&0/30; 30000&0/30; 30000\\
16&100k&0/30; 100000&0/30; 100000&0/30; 100000&0/30; 100000&0/30; 100000\\
\bottomrule
\end{tabular}
\end{table}

%% file: results/load/final_quality_table.tex
\begin{table}[H]
\centering\small
\caption{Quality-shift replay under a fixed two-epoch budget (30 paired trajectories). ``Init'' is epoch~0, ``update'' starts at notification, and ``total'' is the per-trajectory sum before averaging. Censored paths remain; CPU is cumulative oracle decision time in seconds.}
\label{tab:r5-quality}
\begin{tabular}{lrrrrrr}
\toprule
\multicolumn{7}{c}{Certification and readiness time (rounds)}\\
Method&Init cert&Init ready&Update cert&Update ready&Total cert&Total ready\\\midrule
RLR&115238&115238&112978&112978&228216&228216\\
RLR-P&90799&90803&105471&105475&196271&196278\\
Width&99305&99305&114025&114025&213330&213330\\
Width-P&92760&92764&109909&109912&202669&202676\\
\bottomrule
\end{tabular}

\vspace{2pt}
\begin{tabular}{lrrr}
\toprule
\multicolumn{4}{c}{Update/total regret and cumulative oracle CPU}\\
Method&Update regret&Total regret&Total CPU (s)\\\midrule
RLR&16296&42844&0.004\\
RLR-P&16443&36266&35.739\\
Width&15923&38212&0.008\\
Width-P&15631&36679&31.254\\
\bottomrule
\end{tabular}
\end{table}

%% file: results/load/final_quality_macros.tex
\newcommand{\TableTwoUpdateCertDiff}{-4,437}
\newcommand{\TableTwoUpdateCertDiffLo}{-8,856}
\newcommand{\TableTwoUpdateCertDiffHi}{-148}
\newcommand{\TableTwoUpdateRegretDiff}{811}
\newcommand{\TableTwoUpdateRegretDiffLo}{-1,023}
\newcommand{\TableTwoUpdateRegretDiffHi}{3,111}
\newcommand{\TableTwoTotalCpuDiff}{4.48}
\newcommand{\TableTwoTotalCpuDiffLo}{1.35}
\newcommand{\TableTwoTotalCpuDiffHi}{7.49}

%% file: sections/load_appendix.tex
\section{Proofs for prefix-monitored residual repair}\label{app:load}

\subsection{Proof of Lemma~\ref{lem:quota}}
Replace each edge $e$ by $x_e$ parallel copies. The resulting bipartite multigraph has maximum degree $\mathsf L(x)$. K\H{o}nig's line-coloring theorem partitions its edges into $\mathsf L(x)$ matchings, which give a feasible issue schedule of that length. Conversely, an endpoint incident to $\sum_{e\ni v}x_e$ requested observations can serve at most one per round, proving necessity. If the last issue is in round $\mathsf L(x)$ and all delays are at most $d$, all prescribed returns are available by round $\mathsf L(x)+d$. Under a common deterministic delay, neither lower bound can be shortened.

\subsection{Proof of T-Prefix}
Index learning observations by issue order within their true evidence lifetime. Fix an edge $e$ and lifetime $g$. If $X_t$ indicates that round $t$ issues an observation assigned to this edge and lifetime, predictability and Assumption~\ref{ass:reward} make
\[
 Z_s(\lambda)=\exp\!\left\{\lambda\sum_{t\le s}X_t(Y_{e,t}-\mu_{e,g})
 -\frac{\lambda^2\sigma^2}{2}\sum_{t\le s}X_t\right\}
\]
a nonnegative supermartingale in the latent filtration. Stop at the earlier of the $n$th such issue, the lifetime's end, and the global record limit. The usual two-sided Chernoff argument gives
\[
 \Prb\!\left(n\text{ observations exist and }
 |\bar Y_{e,g,n}-\mu_{e,g}|>
 \sigma\sqrt{2\Lambda_{\rm conf}/n}\right)
 \le 2e^{-\Lambda_{\rm conf}}.
\]
There are at most $|\cE|T_{\rm conf}^2$ edge--lifetime--prefix triples, so their union has probability at most $\delta$. On this single event every existing issue-ordered prefix is valid simultaneously. The random complete-prefix count at a report time may depend on the rewards through the delay channel, but substituting it into an already simultaneous event preserves coverage. By contrast, the set of observations that happened to arrive need not be a prefix and is not covered. Intersecting valid intervals cannot remove the true mean.

If a fixed record bound is unavailable and lifetime indices are exogenous, one may instead use
\[
 \Lambda_{g,n}=\log\frac{\pi^4\max\{1,|\cE|\}g^2n^2}{18\delta},
 \qquad r_g(n)=\sigma\sqrt{2\Lambda_{g,n}/n}.
\]
Indeed, summing $2e^{-\Lambda_{g,n}}$ over edges, lifetimes, and prefixes gives exactly at most $\delta$. This anytime alternative does not automatically cover reward-adaptive reset times. It is also distinct from the empirical Bernstein construction, which assumes i.i.d. bounded observations within a lifetime.

\subsection{Proof of T-Monitor} \label{app:a3}
Use the common environment-indexed rewards and issue-time report dates stated in Theorem~\ref{thm:monitor}; statistical invalidation changes a lifetime label but never replays the environment stream. Until the first successful monitor call, the monitor cannot modify the canonical evidence, phase target, decomposition, or block order. Inducting over issue and arrival events therefore shows that RLR-P and RLR have identical issued prefixes and identical canonical states at every completed barrier reached by both.

The extra oracle calls cannot change a later answer because the oracle is a deterministic, stateless function of its input box. If no monitor call succeeds first, both algorithms consequently make the same original call on the same canonical box at each barrier, including when the oracle's acceptance set is nonmonotone under interval intersection. If a monitor call succeeds first, RLR-P cancels only the unissued suffix of the current plan. Its issue sequence is then a prefix of RLR's, and its certification time and issued count are no larger. The same event-prefix argument applies at a finite cap.

Every RLR-P job has the same report time as its RLR counterpart and RLR-P adds no job after certification. Draining its subset therefore finishes no later than RLR's eventual certified and drained run. Finally, on the event of Theorem~\ref{thm:prefix}, both monitor and canonical boxes contain the true mean vector. Soundness of the oracle makes every accepted matching correct. This proof fails for a stateful oracle whose extra calls mutate later behavior, for shared random bits consumed in call order, or for a reporting rule that retroactively changes already-issued jobs in response to later service.

\subsection{Dyadic load, waiting, and lifetime proofs}
For $b=2^J$, let $z_s=2^s$, $z_{-1}=0$, and
\(x_s(e)=[z_s-m_e]_+-[z_{s-1}-m_e]_+\), with $D_e=[b-m_e]_+$. For $s=0$, $x_0(e)\le D_e/b$. For $z_s\ge2$, putting $\alpha_z=z/(2b-z)$, a three-case check on $m_e\ge z$, $m_e\le z/2$, and $z/2<m_e<z$ gives $x_s(e)\le\alpha_zD_e$. Hence
\begin{equation}
 \sum_{s=0}^J\mathsf L(x_s)\le\left(\frac1b+\sum_{k=1}^J\frac1{2^k-1}\right)\mathsf L(D)=c_b\mathsf L(D).
 \label{eq:dyadic-load-app}
\end{equation}
Since $\sum_sx_s=D$, subadditivity gives the reverse inequality $\mathsf L(D)\le\sum_s\mathsf L(x_s)$. The infinite tail is bounded by $2^{-10}/(1-2^{-11})$ after ten terms, yielding $c_b<c_\star<1.607$. Tightness follows from disjoint requester-centered stars: use histories $m=0,z/2$ and choose the number of leaves proportional to $1/(b-m)$ so every phase reaches its coordinate bound. For $b=4$, three stars with histories $0,1,2$ give ratio $19/12$.

For completeness, the same schedule also admits a history-independent envelope.
For every initial count vector $m$ and every phase,
\[
 0\le x_s(e)=[z_s-m_e]_+-[z_{s-1}-m_e]_+
 \le z_s-z_{s-1}.
\]
Monotonicity and positive homogeneity of the endpoint-load functional
$\mathsf L$ give, with $\Delta_E=\mathsf L(\mathbf 1_E)$,
\[
 \mathsf L(x_s)\le
 (z_s-z_{s-1})\mathsf L(\mathbf 1_E)
 =(z_s-z_{s-1})\Delta_E.
\]
Summing over $s=0,\ldots,J$ telescopes to
\[
 \sum_s\mathsf L(x_s)\le b\Delta_E.
\]
Combining this envelope with Eq.~\eqref{eq:dyadic-load-app} yields
\[
 \sum_s\mathsf L(x_s)
 \le \min\{c_b\mathsf L(D),\,b\Delta_E\}.
\]
Adding at most $dB_h$ report-drain rounds gives
Eq.~\eqref{eq:residual-rounds}. Consequently,
$b_h\Delta_E+dB_h$ is a sufficient history-independent dwell condition,
whereas $c_{b_h}\mathsf L(D_h)+dB_h$ is the history-sensitive repair bound.

Fix a repair with initial count $m_e=n^0_{e,h}$ and terminal sufficient target $b=b_h$. At target $z=1$, the increment is $x_1(e)=[1-m_e]_+$. At a later dyadic target $z$, it is
\[
 x_z(e)=[z-m_e]_+-[z/2-m_e]_+.
\]
At $q_h$ observations, every interval has width at most $\rho_h$; hence the canonical oracle accepts no later than $b_h$. RLR follows the full schedule above. RLR-P either does the same or deletes an unissued suffix, so its total issue rounds obey the sharp bound in Eq.~\eqref{eq:dyadic-load-app}. At most $B_h$ nonempty phases are reached and each completed or partially issued final phase drains within $d$ additional rounds. Together with the two envelopes above, this proves Eq.~\eqref{eq:residual-rounds} and its stated dwell conditions.

For the sharper wait calculation, measure issue rounds from one phase's start. The last return occurs at relative time $\max_{e,k}(u_{e,k}+D_{e,k})$, while issuing ends at $Q_s$. This proves the equality in Eq.~\eqref{eq:fine-wait}. Since $u_{e,k}\le l_e$ and $D_{e,k}\le d_e$, subtracting $Q_s$ proves its first inequality; the second is immediate. For an early-stop phase, $Q_s$, $A_s$, and the maximum range only over observations actually issued. Thus a slow edge with zero residual demand contributes no delay term, and issue rounds can hide part of the delay. No claim is made that the chosen block ordering minimizes this expression.

Now fix one edge lifetime $g$. At a drained phase boundary, every issued observation in that lifetime is part of the retained complete record. A phase with target $z$ issues at most the deficit from that record to $z$, and an early stop issues only a prefix of that deficit. Across repairs, lower targets cause no new issue and higher targets telescope from the retained count. Therefore total new issues are at most $[b^{\max}_{e,g}-n^0_{e,g}]_+$. A final truncated phase still issues only part of a target no larger than $b^{\max}_{e,g}$, so the same issue-count bound holds even if its last reports are unavailable at the horizon. Summing proves Eq.~\eqref{eq:load-lifetime}. Discarding pending observations at every early stop would invalidate this telescope and can repay up to a delay-sized suffix after every update.

\subsection{Service allocation and capped-loss proof}\label{Sec:A1-5}
If the exploratory matching $A_t$ has $m_t$ edges, insertion into $M^{\rm old}$ changes the assignments of at most $2m_t$ additional requesters. Since the incumbent differs from $M_h^\star$ on $a_h$ requesters, at most $a_h+2m_t$ assignments differ from the target before certification. Utilities in $[-1,1]$ give a positive-part regret charge at most $2a_h+4m_t$. A nonexploratory pre-certificate round has $m_t=0$; after certification, including during drain, the algorithm serves $M_h^\star$ on the confidence event. Summing through the capped certificate time and using $\sum_tm_t=P_h$ gives Eq.~\eqref{eq:load-regret-cert}.

Only requesters whose assignments differ can create new individual-rationality violations. Any blocking pair that was absent from $M_h^\star$ must involve one of those requesters or a provider whose old or new occupant changed; the latter provider set has size at most twice the changed-requester set. This gives at most $(K+2N+1)(a_h+2m_t)$ violations and proves Eq.~\eqref{eq:load-stability-cert}. On the complement of the simultaneous confidence event, use the universal caps $2NT$ and $N(K+1)T$, then take expectations. If a benchmark starts from an arbitrary incumbent rather than the preceding correct epoch outcome, $a_h$ must instead count its actual assignment differences from $M_h^\star$.

\begin{proposition}[Exact phase-bottleneck overhead]\label{prop:r5-bottleneck}
For phase demands $x_1,\ldots,x_S$, let $D=\sum_sx_s$, $a_{v,s}=\sum_{e\ni v}x_s(e)$ and $q_s=\mathsf L(x_s)$. Then
\[
\sum_sq_s=\mathsf L(D)+\Omega,\qquad
\Omega=\min_v\sum_s(q_s-a_{v,s})\ge0.
\]
Equality holds exactly when one endpoint maximizes every phase. For the complete dyadic schedule, $\Omega\le(c_b-1)\mathsf L(D)$.
\end{proposition}
\begin{proof}
For each $v$, summing $q_s-a_{v,s}$ and taking the minimum gives
$\sum_sq_s-\max_v\sum_sa_{v,s}$. Nonnegativity and the equality condition
follow term by term; the final bound is the sharp dyadic inequality minus
$\mathsf L(D)$.
\end{proof}

\begin{proposition}[Identifiable monitoring suffix]\label{prop:r5-suffix}
Under the single-repair coupling, let a nonintervening shadow monitor first
accept at $\tau_m$, put $t_* =\min\{\tau_m,\tau_R^{\rm cert}\}$ and
$u=\min\{t_*,B\}$. If $A_t$ is the canonical learning matching (empty after
certification), then $P_R(B)-P_P(B)=\sum_{t=u+1}^B|A_t|$. If the monitored
output is correct, the regret difference is the corresponding sum of the
canonical suffix losses and is nonnegative.
\end{proposition}
\begin{proof}
The paths agree through the first acceptance round; only subsequent unissued
rounds are deleted. After acceptance the monitored policy issues no learning
jobs and serves the true target on the stated event. Subtracting the two
finite path sums proves both identities.
\end{proof}

\begin{proposition}[History reduces a fixed load]\label{prop:r5-history}
For a prescribed residual $D$, let valid added history remove
$a_e=\min\{D_e,h_e\}$ and put $D'=D-a$. If
$\ell_v=\sum_{e\ni v}D_e$, $s_v=\mathsf L(D)-\ell_v$, and
$h_v=\sum_{e\ni v}a_e$, then
\[
\mathsf L(D)-\mathsf L(D')=\min_v(s_v+h_v).
\]
The load decreases strictly exactly when every endpoint maximizing
$\mathsf L(D)$ loses positive residual demand.
\end{proposition}
\begin{proof}
Since $\mathsf L(D')=\max_v(\ell_v-h_v)$, subtracting from $\mathsf L(D)$ gives
the identity. The strictness criterion follows because $s_v=0$ precisely at
the original maximizers and all other slacks are positive.
\end{proof}

\begin{proposition}[Exact-query cache equivalence]\label{prop:r5-memo}
Let a deterministic pure oracle be $O(\xi)$ with complete input $\xi$. Reusing
an immutable answer only for exactly equal $\xi$, while retaining every logical
query, preserves oracle responses, actions, certification, readiness, and loss.
The statement does not allow width-, count-, rounded-, or skipped-box keys.
\end{proposition}
\begin{proof}
Induct over logical query events. A hit returns the same pure function value as
recomputation; a miss evaluates the identical input. The cache changes no
archive or scheduler state, so all subsequent events coincide.
\end{proof}

\subsection{Correct restricted costs and paired evaluation}\label{app:r5-restricted}
For an algorithm $A$, let $T_A$ equal its first accepted-output time when
that output is correct and $+\infty$ otherwise, and define
$X_A(B)=\min\{T_A,B\}$. Pathwise,
\[
 X_A(B)=\sum_{t=0}^{B-1}\mathbf 1\{T_A>t\}.
\]
Hence, if $F_A(t)=\Pr(T_A\le t)$ denotes the correct-completion curve,
\[
 \mathbb E[X_A(B)]=\sum_{t=0}^{B-1}\bigl(1-F_A(t)\bigr).
\]
Thus the restricted mean is the area under the non-completion curve,
equivalently the complement of the correct-completion curve, rather than the
area under $F_A$ itself. A fully served capped path therefore contributes $B$
rather than a missing value. A process failure before $B$ is incomplete and is
retained with its identification interval; no value is constructed from
printed PDF means. Fixed-budget regret is identified whenever the service-block
ledger reaches $B$ and is otherwise reported with its observed prefix and
deterministic loss bound.

\section{Earlier public-gap fallback benchmark}\label{app:known-gap}
This section retains the earlier result as a separate benchmark. Give each requester a private provider of known utility $1/2$ and an arbitrary graph of candidate providers. Candidate $e$ has a public magnitude $0<\Delta_e\le1/4$. The hypothesis class contains the all-negative reference mean $1/2-\Delta_e$ on every candidate and one alternative per edge in which only that edge has mean $1/2+\Delta_e$. Thus at most one candidate is positive in any hypothesis. Initial observations are independent current-distribution samples of deterministic common counts $m_e(\delta)$, no historical report is pending, and all hypotheses use a common conditional reporting channel.

Write $\kappa_e=\operatorname{kl}(1/2-\Delta_e,1/2+\Delta_e)$ and $\ell=\log(1/\delta)$. On the reference market, change of measure gives
\begin{equation}
 \E\tau\ge\mathsf L\left(\left[
 \frac{\operatorname{kl}(1-\delta,\delta)}{\kappa_e}-m_e(\delta)
 \right]_+\right).\label{eq:kl-fallback}
\end{equation}

\begin{theorem}[Public-gap first-order fallback repair]\label{thm:likelihood}
Fix a finite candidate graph and its public gaps. If $m_e(\delta)/\ell\to c_e\in[0,\infty)$ and $d(\delta)=o(\ell)$, there is a uniformly $\delta$-correct procedure on this hypothesis class whose reference-market time satisfies
\begin{equation}
 \lim_{\delta\downarrow0}\frac{\E\tau}{\ell}
 =\mathsf L\bigl(( [\kappa_e^{-1}-c_e]_+ )_e\bigr).
 \label{eq:exact-load-limit}
\end{equation}
The right side lower-bounds every uniformly correct procedure's limit inferior. If it is zero, the conclusion is $\E\tau=o(\ell)$.
\end{theorem}
This result assumes public two-point means and analyzes the all-negative reference instance. It is neither the unknown-gap Priority-Search theorem nor an optimality result for RLR-P.

\input{sections/load_finite_benchmark}

\subsection{Full finite-confidence fallback proof}
Candidate providers are distinct from the private providers, each of which accepts only its associated requester. On the reference market every requester strictly prefers its fallback, so the private allocation is uniquely stable. In a single-positive alternative, the corresponding requester and candidate provider must match and every other requester keeps its fallback. Thus each alternative changes the target outcome regardless of candidate priorities.

For alternative $e$, initial KL equals $m_e\kappa_e$ and all other observations have identical laws. The shared-channel change-of-measure inequality gives
$(m_e+\E N_e^{\rm new})\kappa_e\ge\operatorname{kl}(1-\delta,\delta)$.
Summing the resulting necessary counts over every endpoint gives Eq.~\eqref{eq:fallback-lower}. This calculation uses the full initial-data distribution and does not condition on a favorable estimate.

For the upper bound, fix $u_e$ from Theorem~\ref{thm:fallback} before observing new samples. If $m_e>u_e$, use only the first $u_e$ initial observations; otherwise add $u_e-m_e$ observations. Hoeffding's inequality gives error at most $\exp(-2u_e\Delta_e^2)\le\delta/(2|E_c|)$ per edge. A union bound makes every sign correct with probability at least $1-\delta$, and Lemma~\ref{lem:quota} schedules exactly the deficits. Finally,
$\kappa_e=2\Delta_e\log((1+2\Delta_e)/(1-2\Delta_e))\le(32/3)\Delta_e^2$
and $\operatorname{kl}(1-\delta,\delta)\ge\tfrac12\log(1/\delta)$ for $\delta\le.1$. These inequalities yield the stated cold-start factor. If $m_e\le a_e/2$, then $a_e-m_e\ge a_e/2$ gives the same order. For larger histories the lower residual can vanish while the sufficient residual stays positive, so no multiplicative guarantee follows.

\subsection{Proof of Theorem~\ref{thm:likelihood}}
Let $p_e^\pm=1/2\pm\Delta_e$, $v_e=\log(p_e^+/p_e^-)$, and
\[
 S_{e,n}=\sum_{s=1}^n\log\frac{f_{p_e^+}(Y_{e,s})}{f_{p_e^-}(Y_{e,s})}
 =(2\textstyle\sum_{s=1}^nY_{e,s}-n)v_e.
\]
Take $A=\log(2|E_c|/\delta)$ and $z_e=\lceil(A+A^{3/4})/\kappa_e\rceil$. At the initial check and after each completed phase, freeze an unresolved edge as positive if $S_{e,n}\ge A$ and negative if $S_{e,n}\le-A$. Schedule unresolved deficits to $z_e,2z_e,4z_e,\ldots$ with one barrier per phase. Once all signs are frozen, output the hypothesis allocation, using a fixed feasible output for an inconsistent sign pattern.

Under $p_e^-$, $\exp(S_{e,n})$ is a mean-one likelihood-ratio martingale; under $p_e^+$ use $\exp(-S_{e,n})$. The maximal inequality bounds a wrong sign by $e^{-A}$, including a historical crossing, and a union bound gives error at most $\delta/2$. The strong law gives almost-sure eventual termination.

On the reference market, an unresolved edge after phase $r$ has been checked at the deterministic potential prefix length $n_{e,r}=\max\{m_e,2^rz_e\}$. Although a previously resolved edge need not actually be sampled to that length, noncompletion is contained in the corresponding union of potential-sequence events. Since $S_{e,n}$ has drift $-n\kappa_e$ and increments in $[-v_e,v_e]$, Hoeffding's inequality gives, for $n\kappa_e>A$,
\[
 \Prb(S_{e,n}>-A)\le
 \exp\left[-\frac{(n\kappa_e-A)^2}{2nv_e^2}\right].
\]
For fixed graph and gaps there are $C,c>0$ such that the probability $U_r$ of noncompletion after phase $r$ obeys
$\Prb(U_0)\le Ce^{-c\sqrt A}$ and
$\Prb(U_r)\le Ce^{-c2^rA}$ for $r\ge1$. Consequently
\[
 \E\tau\le\mathsf L([z-m]_+)+d
 +\sum_{r\ge1}(2^r\mathsf L(z)+d)\Prb(U_{r-1})
 =\mathsf L([z-m]_+)+o(A).
\]
This summable tail proves an expectation bound. Since $A/\ell\to1$, $z_e/\ell\to\kappa_e^{-1}$, $m_e/\ell\to c_e$, and $d/\ell\to0$, continuity of the finite-dimensional load map gives the upper limit in Eq.~\eqref{eq:exact-load-limit}. Equation~\eqref{eq:kl-fallback} and $\operatorname{kl}(1-\delta,\delta)/\ell\to1$ give the reverse limit inferior, including critical and oversupplied histories.

\input{sections/priority_star_appendix}

%% file: sections/load_finite_benchmark.tex
\subsection{Finite-confidence fallback benchmark}
\begin{theorem}[Information and load in fallback markets]\label{thm:fallback}
Suppose $m_e$ independent completed observations from the current candidate distribution are available initially, with deterministic counts. Invalidated edges have $m_e=0$, and no historical reward is pending. Let $E_c$ be the candidate edges and
\begin{equation}
 a_e=\frac{\operatorname{kl}(1-\delta,\delta)}
 {\operatorname{kl}(1/2-\Delta_e,1/2+\Delta_e)},\qquad
 u_e=\left\lceil\frac{\log(2|E_c|/\delta)}{2\Delta_e^2}\right\rceil.
\end{equation}
On the reference market, every uniformly $\delta$-correct algorithm satisfies
\begin{equation}
 \E\tau\geq\mathsf L(([a_e-m_e]_+)_e).\label{eq:fallback-lower}
\end{equation}
A one-phase procedure, using the public magnitudes and the first $u_e$ current observations per candidate to test which side of $1/2$ its mean lies, is uniformly $\delta$-correct and uses at most
\begin{equation}
 \mathsf L(([u_e-m_e]_+)_e)+d\,\ind\{\exists e:m_e<u_e\}\label{eq:fallback-upper}
\end{equation}
rounds. For $m=0$, $\delta\leq.1$, and immediate feedback, the upper/lower ratio is $O(\log(|E_c|/\delta)/\log(1/\delta))$. The same order holds if $m_e\leq a_e/2$ for every edge. No multiplicative statement is made for arbitrary warm starts.
\end{theorem}
\begin{proof}[Proof sketch]
Changing the sign of candidate $e$ changes the unique allocation. Its initial KL is exactly $m_e\kappa_e$, where $\kappa_e$ is the denominator of $a_e$. Adaptive change of measure requires $(m_e+\E N_e)\kappa_e\geq\operatorname{kl}(1-\delta,\delta)$. Summing necessary new counts at each endpoint gives the lower bound. For the upper bound, Hoeffding's inequality and a union bound make all candidate signs correct at their fixed targets $u_e$; Lemma~\ref{lem:quota} schedules precisely the remaining observations. Since $\kappa_e=\Theta(\Delta_e^2)$, the cold-start ratio follows. For moderately warm histories, $a_e-m_e\geq a_e/2$ preserves it. The full proof is in Appendix~\ref{app:load}.
\end{proof}
The class embeds a threshold-identification problem \citep{locatelli2016thresholding}; it is a controlled achievability result, not a characterization of arbitrary stable matching. It proves that both endpoints and the location of valid history matter, and that reset can be arbitrarily wasteful when only a dispersed subset needs new evidence. The statistical constants differ between $a_e$ and $u_e$: a zero lower bound does not imply that the sufficient quota is zero. Public gap magnitudes are inputs only to this benchmark procedure, not to RLR.

%% file: sections/priority_star_appendix.tex
\section{Priority-star theorem: full proof}\label{app:priority-star}

\subsection{Market structure}
Index the requesters in group $j$ by the shared provider's strict priority, $1,\ldots,r_j$. Each requester has exactly two feasible providers: its private fallback, which accepts only that requester and gives known utility $\theta$, and the group's shared provider. Groups share neither requesters nor providers.

\begin{lemma}[Priority-star outcome]\label{lem:priority-star-structure}
Every parameter vector with $p_i\ne\theta$ has a unique stable matching. In group $j$, its shared provider is assigned to the first requester with $p_i>\theta$, if one exists; every other requester receives its private fallback.
\end{lemma}
\begin{proof}
Any requester not assigned the shared provider must receive its private fallback, since otherwise its positive-utility private edge blocks. If every candidate mean is below $\theta$, no requester is willing to take the shared provider. Otherwise let $a_j$ be the highest-priority requester with $p_i>\theta$. It cannot stably remain at its fallback while the shared provider is empty or assigned to a lower-priority requester. A higher-priority requester cannot stably hold the shared provider because that requester prefers its fallback. The stated assignment is therefore necessary. It is stable because higher-priority requesters do not want the shared provider and lower-priority requesters cannot displace $a_j$. Applying this argument independently to every group proves existence and uniqueness.
\end{proof}

The outcome-relevant set is exactly the prefix $\mathcal B_j$. If $i<a_j$, changing only $p_i$ to any $\lambda_i>\theta$ makes $i$ the new shared-provider assignee. If $i=a_j$, changing it to any $\lambda_i<\theta$ removes $i$ and exposes a later positive requester or leaves the provider empty. When $a_j=\infty$, changing any candidate above $\theta$ creates an assignee. Thus every $i\in\mathcal B_j$ has a one-coordinate, different-outcome alternative. Coordinates after a finite $a_j$ need not be identified.

\subsection{Information-load lower bound}
Fix a true instance $p$, a group $j$, and $i\in\mathcal B_j$. Choose the preceding alternative $\lambda^{(i)}$ on the opposite side of $\theta$. The deterministic history design is identical under the two hypotheses, so its KL divergence is
$m_i(\delta)\operatorname{kl}(p_i,\lambda_i)$.

For completeness, enlarge the delayed transcript by revealing each new reward at issue and then simulating its reported value and report time through the common conditional channel. The learner's actual transcript is a measurable degradation of this augmented experiment. If $N_i$ is the number of new observations issued on $i$ by certification, the chain rule and data processing give
\begin{equation}
 \bigl(m_i(\delta)+\E_pN_i\bigr)
 \operatorname{kl}(p_i,\lambda_i)
 \ge \operatorname{kl}(1-\delta,\delta). \label{eq:star-change-measure}
\end{equation}
Indeed, the event that the learner outputs the true $p$-matching has probability at least $1-\delta$ under $p$ and at most $\delta$ under the alternative, whose unique matching differs. This argument counts pending issued rewards on the more informative side of data processing; such rewards cannot make the lower bound invalid. It also shows why a parameter-dependent reporting channel or history count could leak information and is excluded.

Letting $\lambda_i$ approach $\theta$ from the other side and using continuity of Bernoulli KL yields
\begin{equation}
 \E_pN_i\ge
 \left[\frac{\operatorname{kl}(1-\delta,\delta)}
 {\operatorname{kl}(p_i,\theta)}-m_i(\delta)\right]_+ . \label{eq:star-count-lb}
\end{equation}
The common shared provider can issue at most one candidate observation per round, hence pathwise $\sum_{i\in\mathcal B_j}N_i\le\tau_\delta^{\rm cert}$. Sum Eq.~\eqref{eq:star-count-lb}, take the maximum over groups, divide by $\ell$, and use $m_i(\delta)/\ell\to c_i$ and $\operatorname{kl}(1-\delta,\delta)/\ell\to1$. This proves the liminf in Theorem~\ref{thm:priority-star}. The quantifiers are pointwise in every fixed nonthreshold instance, while correctness is uniform over the entire priority-star class.

\subsection{A time-uniform, mean-blind sign test}
Let the full potential sequence on candidate $i$ be $Y_{i,1},Y_{i,2},\ldots$, including its deterministic-length historical prefix. Define
\begin{equation}
 T_{i,\delta}=\inf\left\{n\ge1:
 n\operatorname{kl}(\widehat p_{i,n},\theta)
 \ge b_\delta(n)\right\},\qquad
 b_\delta(n)=\log\frac{2Mn^2}{\delta}. \label{eq:star-stopping}
\end{equation}
At a crossing, declare positive if $\widehat p_{i,n}>\theta$ and negative if it is below $\theta$; equality has zero statistic. Bernoulli KL uses its continuous values at empirical means zero and one. The procedure scans every historical prefix in order. With out-of-order reports, it caches returns and scans every newly completed prefix when a hole closes, rather than inspecting only the final count.

Suppose $p_i<\theta$. At a fixed $n$, a wrong positive crossing is an upper-tail event with an empirical mean $x>\theta$. Monotonicity gives
$\operatorname{kl}(x,p_i)\ge\operatorname{kl}(x,\theta)$, so the Bernoulli Chernoff bound makes its probability at most $e^{-b_\delta(n)}$. The lower-tail argument is symmetric when $p_i>\theta$. Therefore
\begin{equation}
 \Prb_p(\text{candidate $i$ ever crosses on the wrong side})
 \le\sum_{n\ge1}e^{-b_\delta(n)}
 =\frac{\pi^2\delta}{12M}<\frac\delta M. \label{eq:star-error}
\end{equation}
A union bound over all $M$ potential tests, including tests never reached by the policy, is below $\delta$. This time-uniform event covers adaptive entry into later priority positions and reward-dependent arrival order.

\subsection{First-passage asymptotics with arbitrary linear history}
Fix one candidate and abbreviate $p=p_i$, $k=\operatorname{kl}(p,\theta)>0$, and $T_\delta=T_{i,\delta}$. Couple all confidence levels on one infinite reward sequence. By the strong law and continuity,
$\operatorname{kl}(\widehat p_n,\theta)\to k$ almost surely. For any $0<\varepsilon<k/2$, all sufficiently large prefixes have KL in $[k-\varepsilon,k+\varepsilon]$. Also, as $\delta\downarrow0$, no fixed finite prefix crosses the diverging boundary. A crossing requires $n(k+\varepsilon)\ge\ell$, while at
$n=\lceil\ell/(k-2\varepsilon)\rceil$ the inequality
$n(k-\varepsilon)\ge\ell+\log(2M)+2\log n$
holds for all sufficiently large $\ell$. Letting $\varepsilon\downarrow0$ proves
\begin{equation}
 T_\delta/\ell\longrightarrow k^{-1}\quad\text{almost surely}. \label{eq:star-as}
\end{equation}

We also need convergence of expectations. Fix $\eta>0$. Continuity provides $a>0$ such that $|x-p|\le a$ keeps $x$ on the true side of $\theta$ and makes $\operatorname{kl}(x,\theta)>k/(1+\eta/2)$. For all sufficiently large $\ell$ and all
$n\ge n_0=\lceil(1+\eta)\ell/k\rceil$, one has
$b_\delta(n)/n<k/(1+\eta/2)$; the ratio is decreasing on this range. If $T_\delta>n$, the $n$th prefix has not crossed, so $|\widehat p_n-p|>a$. Hoeffding's inequality gives
\[
 \Prb(T_\delta>n)\le2e^{-2a^2n},\qquad n\ge n_0.
\]
The discrete tail identity then gives $\E T_\delta^2=O(\ell^2)$. Thus $T_\delta/\ell$ is uniformly integrable and Eq.~\eqref{eq:star-as} also holds in $L^1$. Since the positive-part map is 1-Lipschitz and $m_i(\delta)/\ell\to c_i$,
\begin{equation}
 \frac{(T_{i,\delta}-m_i(\delta))_+}{\ell}
 \longrightarrow
 \left(\operatorname{kl}(p_i,\theta)^{-1}-c_i\right)_+
 \quad\text{in }L^1. \label{eq:star-residual-l1}
\end{equation}
This includes undersupplied, critical, and oversupplied histories.

\subsection{Priority-Search execution and upper bound}
Within each group, Priority-Search activates candidates in public priority order. It first scans an edge's complete history; absent a crossing, it issues that edge once per access round until a complete prefix crosses. A negative decision activates the next candidate and the first positive decision ends the group. Different groups issue their active candidates simultaneously. Their requester and provider endpoints are disjoint, so each round is a physical matching. No $p_i$, gap, sign, or true prefix $\mathcal B_j$ is an algorithm input.

On the complement of the error event in Eq.~\eqref{eq:star-error}, a group visits exactly $\mathcal B_j$ and Lemma~\ref{lem:priority-star-structure} makes its output correct. Hence the procedure is uniformly $\delta$-correct. Every fixed nonthreshold test stops almost surely by Eq.~\eqref{eq:star-as}, so the finite graph finishes almost surely.

If $T_{i,\delta}>m_i$, continuously issuing the missing prefix and using the delay bound makes it available within
\begin{equation}
 Z_i=(T_{i,\delta}-m_i)_+
 +d(\delta)\ind\{T_{i,\delta}>m_i\}. \label{eq:star-stage}
\end{equation}
rounds. Up to $d$ later-issued observations may remain pending at the crossing. They are charged as issues, but by assumption they do not occupy the shared provider's later access slots. Thus the next candidate may start immediately. A final drain, if needed to export a reusable checkpoint, adds at most $d=o(\ell)$ but is not part of certification time.

The Figure~\ref{fig:priority-star} delay audit uses the same coupled potential streams and the two deterministic regimes implemented by Priority-Search: $d=0$ and $d=\lfloor\sqrt{\ell}\rfloor$. Delay changes report availability and therefore the first-passage round, but never the reward sequence or the ordinate by a coordinate shift. The plotted per-seed differences are retained in the R6 raw map; they are a diagnostic of this finite-run schedule and do not change the asymptotic constant because $d=o(\ell)$.

It remains to control rare wrong decisions in expectation. Let $\mathcal E_\delta$ be the event that any potential test ever crosses on the wrong side; then $\Prb(\mathcal E_\delta)<\delta$. Pathwise,
\begin{equation}
 \tau_\delta^{\rm cert}\le
 \max_j\sum_{i\in\mathcal B_j}Z_i
 +\ind\{\mathcal E_\delta\}
 \sum_{i\notin\cup_j\mathcal B_j}Z_i. \label{eq:star-path-upper}
\end{equation}
Equation~\eqref{eq:star-residual-l1}, the second-moment bound above, and $d=o(\ell)$ give $\E Z_i^2=O(\ell^2)$. Since the graph is fixed, Cauchy--Schwarz bounds the expectation of the second term in Eq.~\eqref{eq:star-path-upper} by
$\sqrt\delta\sum_i\sqrt{\E Z_i^2}=o(\ell)$, without any independence assumption between errors and stopping times. For the first term, finite sums and maxima are Lipschitz, so Eq.~\eqref{eq:star-residual-l1} implies
\[
 \frac1\ell\E_p\left[\max_j\sum_{i\in\mathcal B_j}Z_i\right]
 \longrightarrow\max_j\sum_{i\in\mathcal B_j}
 \left(\operatorname{kl}(p_i,\theta)^{-1}-c_i\right)_+.
\]
Together with the lower bound, this proves the equality in Theorem~\ref{thm:priority-star}.

The proof relies on a fixed finite graph, i.i.d. Bernoulli candidate streams, a known fallback threshold, deterministic hypothesis-independent history sizes, a shared reporting channel, nonoccupying pending work, and $d=o(\log(1/\delta))$. It does not cover growing graphs, threshold means, adaptively encoded histories, delay of logarithmic order, or general matching markets. General combinatorial pure exploration and batch complexity address broader statistical allocation questions \citep{jourdan2021combinatorial,tuynman2025batch}; T-Star identifies the constant only for the stated matching structure.

%% file: sections/load_reference.tex
\section{Reference certificates and width-greedy repair}\label{sec:outcome}
A learning rule need not determine a unique DA execution to determine its outcome. We first expose this distinction, then separate the certificate from the sampling and reuse mechanism.

\begin{proposition}[An irrelevant winning comparison]\label{prop:irrelevant-winner}
Let both providers rank requester 2 above 1 and let
\[
\mu(\eta)=\begin{pmatrix}0.6+\eta&0.6\\0.9&0.1\end{pmatrix},\qquad |\eta|\leq0.1.
\]
Every market has the unique stable matching $(1b,2a)$. The smallest-ID proposal trace nevertheless has margin $|\eta|$ when $\eta\ne0$, and its first strict query cannot be certified at $\eta=0$ on a sound interval box.
\end{proposition}
Indeed, $(2,a)$ blocks any outcome not containing it, and requester 1 then takes the remaining acceptable provider. This separation concerns our trace certificate, not every previous partial-preference algorithm. A diverging trace-based upper bound alone is not a sample-complexity lower bound for all algorithms.

\paragraph{A statistical separation between objectives.}
The two markets $\mu(+\eta)$ and $\mu(-\eta)$ have different first canonical proposals but the same outcome. With independent Bernoulli rewards, no initial information and a common reporting kernel, any procedure identifying that proposal with error at most $\delta\leq.1$ in both markets needs $\E N_{1a}\geq\log(1/\delta)/(64\eta^2)$ for $0<\eta\leq.1$. Only edge $(1,a)$ differs: transcript data processing requires binary relative entropy $\operatorname{kl}(1-\delta,\delta)$, while its per-pull KL is at most $32\eta^2$. In contrast, requester 2's strict preference and requester 1's acceptability of $b$ certify the allocation at fixed precision. Appendix~\ref{app:exclusion} gives the full argument. This separates trace identification from outcome certification; existing outcome-based methods can share the latter advantage.

\paragraph{A polynomial outcome certificate.}
A \emph{sound outcome oracle} returns either unresolved or a matching that equals requester-optimal DA for every utility matrix in the supplied interval box. Our implemented oracle first checks the original trace certificate (Appendix~\ref{sec:trace}). If unresolved, it repeatedly removes a \emph{certified mutual first-choice pair}: requester $i$ has a strictly certified best acceptable provider $j$ in the residual market, and $j$ ranks $i$ above every remaining feasible requester. It may also remove a requester whose remaining upper bounds are all negative. It then checks the residual trace and reinserts the removed pairs. Each pass uses public priorities and intervals only; lowest requester/provider IDs break choices between eligible removals.

\begin{theorem}[Sound forced-pair reduction]\label{thm:outcome}
The resulting oracle is sound, costs $O((N+1)(|\cE|+N)(K+1))$ comparisons, and can certify Proposition~\ref{prop:irrelevant-winner} without resolving requester 1's initial ordering. In that family, every sound box with maximum edge width at most $0.1$ is certified, uniformly over $|\eta|\leq0.1$.
\end{theorem}
The contraction proof is in Appendix~\ref{app:modular}.
Mutual first-choice reduction is a classical matching argument, not claimed as a new solution concept. Our oracle remains sufficient rather than complete; stronger partial-preference certificates already exist \citep{athanasopoulos2026probably,basu2025superstable}.

\paragraph{An explicit exclusion margin.}
Forced pairs do not exhaust outcome information. For example, let
\begin{equation}
\mu=\begin{pmatrix}.8&.5&.5\\.2&.5&.8\\.2&.8&.8\end{pmatrix},\qquad
\succ_a:3>1>2,\quad\succ_b:3>2>1,\quad\succ_c:1>2>3.\label{eq:cyclic}
\end{equation}
Its sole weakly stable matching is $(1a,2c,3b)$, yet no mutual first pair exists and the residual trace requests the tied comparison $3b$ versus $3c$, even at exact means. We call a feasible matching weakly stable when it is individually rational and has no pair with a strict requester improvement and a willing provider, as in $U_0$.

For any feasible $M$, let $A(M)=\{(i,j)\in\cE:j\ne M(i),\ i\succ_jM^{-1}(j)\}$. Set $\ell_{i0}=u_{i0}=\mu_{i0}=0$. Define its exclusion score and the market margin by
\begin{align}
v_\mu(M)&=\max\bigl(\{0\}\cup\{-\mu_{i,M(i)}:i\}\cup
 \{\mu_{ij}-\mu_{i,M(i)}:(i,j)\in A(M)\}\bigr),\nonumber\\
\beta(\mu)&=\min_{M\ne M^\star}v_\mu(M).\label{eq:exclusion-gap}
\end{align}
The minimum is over feasible matchings; use $\beta=1$ if there is no alternative.

\begin{theorem}[Possible stability and finite resolution]\label{thm:exclusion}
A feasible matching $M$ is weakly stable for some completion of a box if and only if
\begin{equation}
u_{i,M(i)}\geq0\quad\forall i,\qquad
\ell_{ij}\leq u_{i,M(i)}\quad\forall(i,j)\in A(M).\label{eq:possible-stable}
\end{equation}
If exactly one matching passes, it is the refined DA outcome throughout the box. Moreover, $\beta(\mu)>0$ if and only if $M^\star$ is the unique weakly stable matching. In that case every sound box of maximum width at most $\beta(\mu)/3$ has exactly one passing matching.
\end{theorem}
\begin{proof}
Necessity follows from individual rationality and the absence of blocking pairs. For sufficiency, assign each matched utility its upper endpoint and each other utility its lower endpoint. All inequalities then hold simultaneously, including those involving the outside option. DA at every completion is weakly stable, so a sole possible matching must equal it. A wrong $M$ has a strictly positive witness in $v_\mu(M)$ precisely when it is not weakly stable. Finiteness gives the equivalence for $\beta$. Shrinking a sound interval by width at most $\beta/3$ preserves an exclusion witness with positive slack: a blocking difference loses at most $2\beta/3$, and an individual-rationality witness at most $\beta/3$.
\end{proof}
Testing Eq.~\eqref{eq:possible-stable} over all feasible matchings costs $O(|\mathcal M|(|\cE|+N))$, where $|\mathcal M|\leq\sum_r\binom Nr\binom Kr r!$. This is exponential, not a scalable replacement for pervasive inference. Our \emph{Hybrid} oracle checks forced pairs first, then this exclusion test. It inherits the positive radius above; multiple weakly stable matchings can still prevent exclusion, even when a refined optimal outcome is identifiable. Equation~\eqref{eq:cyclic} has $\beta=.3$; perturbing $\mu_{3b}$ by $|\eta|\leq.01$ preserves its unique outcome and $\beta\geq.29$. Thus finite resolution does not require resolving that tie. The published pervasive oracle also resolves this family; our contribution here is the explicit margin and its connection to reusable-evidence accounting, not priority over that oracle.

\paragraph{Outcome-guided parallel repair (O-CGR).}
Maintain chronological counts and sound intervals on each edge, resetting only the announced edges. If the oracle is unresolved, let $w_{\max}$ be the largest interval width. Sort edges by decreasing width, then IDs, and greedily select a matching from edges with width at least $w_{\max}/2$. Double each selected edge's count (the first batch has one pull). Execute these independent edge batches concurrently; shorter batches finish earlier. Wait until every issued reward in the selected group has arrived before updating any interval. No next group is launched while feedback is pending. This is a feasibility-based parallel schedule, not an optimal allocation of experiments.

At each exploratory round, insert the active selected edges into the old matching, removing incident old pairs. At most $2|S_t|$ requester assignments change. Execute the old matching while waiting and the certified new matching after completion. Nonselected rewards are ignored. The serial ablation selects only the widest edge. Unlike the original query-directed CGR, O-CGR may sample edges irrelevant to its eventual outcome certificate; this tradeoff allows a certificate-independent guarantee.

\begin{algorithm}[t]
\caption{Outcome-guided parallel repair (O-CGR)}\label{alg:outcome}
\KwIn{public graph/priorities, $T,\sigma,\delta$; notifications $\mathcal C_h$; sound oracle $\mathsf O$}
At notification, invalidate only $\mathcal C_h$; retain other statistics\;
\While{$\mathsf O$ is unresolved on the current interval box}{
 Compute the maximum edge width $w_{\max}$\;
 \If{$w_{\max}=0$}{Stop this repair as unresolved; retain $M^{\rm old}$ until the boundary\;}
 Greedily select disjoint edges with width at least $w_{\max}/2$\;
 Fix all selected batch sizes at $\max\{1,n_e\}$ additional observations\;
 Concurrently issue these batches, inserting active edges into $M^{\rm old}$\;
 Execute $M^{\rm old}$ until all selected rewards arrive\;
 Update each selected edge using its entire chronological sample prefix\;
}
If certified, execute the outcome until the next notification; otherwise use the boundary fallback\;
\end{algorithm}

\section{Reference lifetime and information results}\label{sec:modular}
For $0<\delta<1$, use the common record horizon $T_{\rm conf}$ (equal to $T$ in the earlier cold-start experiments)
\begin{equation}
\Lambda=\log\frac{2\max\{1,|\cE|\}T_{\rm conf}^2}{\delta},\qquad r(n)=\sigma\sqrt{2\Lambda/n},\label{eq:cs}
\end{equation}
intersecting $[\widehat\mu_e-r(n),\widehat\mu_e+r(n)]\cap[-1,1]$ with previous intervals in the same lifetime. Start with $[-1,1]$ or a valid prior. An empty intersection off the confidence event falls back to the latest clipped interval. The complete-batch concentration argument applies to simultaneous selected edges: each edge's estimator still uses an initial chronological prefix, and its issuance decision is predictable. It permits bounded reward-dependent delays without assuming that the early arrivals are representative.

\paragraph{Resolution parameters.}
For a sound oracle, a number $\rho_h>0$ is a \emph{resolution radius} if the oracle certifies every box containing $\mu_h$ whose maximum edge width is at most $\rho_h$. It is an analysis parameter, not an input. It is oracle-specific and need not equal an information-theoretic gap. A positive trace margin gives a positive radius for an oracle that first checks the trace; Theorem~\ref{thm:exclusion} supplies the explicit choice $\rho_h=\beta(\mu_h)/3$ for Hybrid on every unique-weakly-stable market, including zero trace margins. Define
\begin{equation}
 q_h=\max\left\{1,\left\lceil64\sigma^2\Lambda/\rho_h^2\right\rceil\right\},\quad
 q_{e,g}=\max_{h\in g}q_h,\quad b_{e,g}=1+\lceil\log_2(2q_{e,g})\rceil.\label{eq:oracle-budget}
\end{equation}
The uniform radius is conservative; unlike the original trace-specific bound it does not identify a minimal set of informative edges.

\begin{theorem}[Certificate-independent lifetime accounting]\label{thm:modular}
Assume positive resolution radii and a nonfinal epoch duration at least
$F_h^{\mathsf O}=\sum_{e\in\cE}\{2q_h+d_e[1+\lceil\log_2(2q_h)\rceil]\}$.
With probability at least $1-\delta$, every completed O-CGR repair is requester-optimal. If $P$ counts exploratory \emph{edge pulls}, $Q$ exploratory \emph{rounds}, and $W$ waiting rounds, then
\begin{equation}
 Q\leq P\leq2\sum_{e,g}q_{e,g},\qquad
 W\leq\sum_{e,g}d_e b_{e,g}.\label{eq:parallelbudget}
\end{equation}
Let $a_h=|\{i:M_{h-1}^\star(i)\ne M_h^\star(i)\}|$, with $M_0^\star$ empty. The realized regret and violation counts satisfy
\begin{align}
 \mathcal R_T&\leq\sum_h\{2a_h(Q_h+W_h)+4P_h\},\label{eq:parallelregret}\\
 \mathcal S_T^0&\leq(K+2N+1)\sum_h\{a_h(Q_h+W_h)+2P_h\}.
\end{align}
Expectation bounds add $2NT\delta$ and $N(K+1)T\delta$. The sums may include only actually sampled lifetimes for the pathwise bounds; all lifetimes give deterministic bounds. The final epoch may be truncated.
\end{theorem}
\begin{proof}
If unresolved, $w_{\max}>\rho_h$. Every selected edge has width $>\rho_h/2$, whereas $n\geq q_h$ implies width at most $2r(n)<\rho_h/2$ when $\sigma>0$. For $\sigma=0$ one completed sample has zero width. Thus a selected edge starts below its threshold. Doubling and retained counts give a maximum threshold per lifetime, rather than paying it again at each epoch. Each group waits at most the largest delay bound among its selected edges, which is bounded by their sum; charge this to edge batch counts. During an exploratory round, the old allocation contributes at most $2a_h$ regret and inserting $|S_t|$ edges contributes at most $4|S_t|$. Sum $|S_t|$ to obtain $P_h$. The stability argument charges the corresponding number of changed assignments. Details and truncation handling appear in Appendix~\ref{app:modular}.
\end{proof}
This theorem applies to other sound outcome oracles with positive resolution radii; selective retention is not exclusive to O-CGR. It provides no universal factor-$N$ speedup and no pointwise dominance over reset policies. For $T\geq2$ and $\delta=T^{-2}$, the regret failure charge is $2N/T$. For fixed $N$, a sublinear deterministic right-hand side in Eq.~\eqref{eq:parallelregret} yields $R_T=o(T)$; frequent updates or vanishing resolution radii can prevent this.

\paragraph{Variance-sensitive evidence for i.i.d. streams.}
The sub-Gaussian radius protects against general martingale rewards but can be conservative on low-variance training gains. Under the additional assumption of an i.i.d. reward stream within each edge lifetime, let $V_{e,n}$ be the unbiased sample variance and $\Lambda_{\rm EB}=\log(4\max\{1,|\cE|\}T_{\rm conf}^2/\delta)$. For $n\geq2$, an alternative is
\begin{equation}
r_{\rm EB}(e,n)=\sqrt{2V_{e,n}\Lambda_{\rm EB}/n}
+\frac{14\Lambda_{\rm EB}}{3(n-1)}.\label{eq:eb-main}
\end{equation}
For $n<2$, keep the prior interval. The two-sided empirical Bernstein inequality \citep{maurer2009empirical}, applied to chronological prefixes and union-bounded over edge lifetimes, gives simultaneous coverage $1-\delta$. Reward-dependent delays are still permitted because a completed batch includes every issued reward. The lifetime proof transfers with a conservative threshold $q_h^{\rm EB}=\max\{2,\lceil512\Lambda_{\rm EB}/\rho_h^2+1+112\Lambda_{\rm EB}/(3\rho_h)\rceil\}$: $n\geq q_h^{\rm EB}$ forces width below $\rho_h/2$. Appendix~\ref{app:exclusion} gives the rescaling and threshold derivation. This optional extension does not replace the martingale assumption of the main theorem or introduce a new concentration inequality.

\begin{theorem}[Information still needed after an update]\label{thm:warm-lower}
Consider two post-update markets with independent Bernoulli edge rewards, different target outcomes, and means differing only on edge $c$, with values $p_+,p_-\in(0,1)$. The two markets share a conditional reporting kernel given issued rewards and the past. No pre-update observation remains pending, or the historical pending-report process has a common conditional kernel given the initial information and the evolving augmented transcript. Let $Z$ contain all information available immediately before the first post-update pull, with laws $Q_+,Q_-$ and $D_0=\operatorname{KL}(Q_+,Q_-)$. Any algorithm correct with error at most $\delta<1/2$ in both markets satisfies
\begin{equation}
 \E_+N_c^{\rm new}\geq
 \frac{[\operatorname{kl}(1-\delta,\delta)-D_0]_+}
 {\operatorname{kl}(p_+,p_-)}.\label{eq:warm-lower}
\end{equation}
Here $N_c^{\rm new}$ counts newly issued pulls before output; other feedback has the same conditional law. In the pivotal construction of Eq.~\eqref{eq:pivotal}, an identical pre-update history gives $D_0=0$ regardless of its length, and the fresh-pull lower bound is $\log(1/\delta)/(64\eta^2)$ for $\delta\leq .1$.
\end{theorem}
\begin{proof}
Reveal each newly issued reward immediately and generate its delayed report through the shared kernel. This augmented experiment simulates the actual transcript. The adaptive change-of-measure argument \citep{kaufmann2016complexity} bounds its KL by $D_0+\E_+N_c^{\rm new}\operatorname{kl}(p_+,p_-)$. The assumption on historical pending reports excludes uncharged information. The event of outputting the plus outcome gives KL at least $\operatorname{kl}(1-\delta,\delta)$ by data processing. Rearrange. The Bernoulli estimate for the pivotal instance is in Appendix~\ref{app:separations}.
\end{proof}
For $L$ announced pivotal updates with identical histories under each current plus/minus alternative, uniform per-repair correctness gives an $L\log(1/\delta)/(64\eta^2)$ fresh-pull lower bound by summation. Arbitrarily much data about an old mean cannot identify an unrestricted new mean. Conversely, unchanged edges need not repay their entire sampling budget. This explains the role of invalidation lifetimes without claiming that our uniform-radius upper bound, or its repeated delay charges, is tight.

\paragraph{Information subject to matching capacity.}
The same argument couples historical information to feasible exploration. Fix Bernoulli means $\mu\in(0,1)^{\cE}$ and a collection $\mathcal A\subset(0,1)^{\cE}$ of different-outcome alternatives satisfying the shared-channel and historical-feedback conditions above. Let $D_0(\lambda)$ be their initial-information KL. Any uniformly $\delta$-correct procedure has expected elapsed time at least
\begin{equation}
\begin{aligned}
\inf_{x\geq0,\,\tau\geq0}\;&\tau\quad\text{subject to}\quad
\sum_jx_{ij}\leq\tau,\quad\sum_ix_{ij}\leq\tau,\\
&\sum_{e\in\cE}x_e\operatorname{kl}(\mu_e,\lambda_e)
\geq[\operatorname{kl}(1-\delta,\delta)-D_0(\lambda)]_+
\quad\forall\lambda\in\mathcal A.
\end{aligned}\label{eq:history-load}
\end{equation}
Indeed, take $x_e=\E_\mu N_e^{\rm new}$. Each alternative gives an information constraint; one assignment per endpoint per round gives the load constraints. Shared endpoints can therefore bottleneck new information even when total pulls are small. This is a necessary allocation bound, not an achievable characterization of delayed stopping time; Eq.~\eqref{eq:parallelbudget} need not attain it.

%% file: sections/outcome_appendix.tex
\section{Details of outcome certification and modular repair}\label{app:modular}
\paragraph{Stable-match correspondence under contraction.}
Consider a certified mutual first-choice pair $(i,j)$ at some residual stage. Both comparisons are strict for every completion of the current interval box, and $i$ regards $j$ as acceptable. Any stable matching omitting $(i,j)$ is blocked by it. Conversely, extend a stable matching of the residual market by $(i,j)$. A putative new blocking pair must involve $i$ or $j$, but each prefers its fixed partner to every remaining feasible counterpart. It cannot block. Previously removed pairs obey the same argument at their earlier removal stage. Therefore contraction and extension give a bijection of stable matchings, preserving the partner ordering of every remaining requester. Applying the same reasoning to each fixed tie refinement establishes requester optimality, not just stability. A requester with all remaining upper bounds negative is necessarily unmatched and cannot form a blocking pair. Its removal is also valid.

The oracle first checks the original trace, so all boxes accepted by that check remain accepted. This is a dominance statement about the acceptance region of the \emph{certificate}, not about the sampler's realized costs. It does not imply that O-CGR's global-width scheduler dominates query-directed CGR. A brute-force enumeration of completions is unnecessary. At each of at most $N$ contractions, scanning potential pairs and their competitors is polynomial; a final residual DA check adds $O((|\cE|+N)(K+1))$ comparisons.

\paragraph{A resolution radius is not a minimal sufficient statistic.}
The definition quantifies over every sound box of sufficiently small maximum width, including asymmetric boxes produced by interval intersection. On a strict canonical trace with minimum comparison gap $\gamma>0$, any width $\rho<\gamma/2$ suffices, because a sound interval's lower endpoint is at least its true mean minus its width and its upper endpoint at most its true mean plus its width. The O-CGR oracle inherits this guarantee through its initial trace check. In Proposition~\ref{prop:irrelevant-winner}, however, $\rho=0.1$ works even at the trace tie: $\ell_{2a}\geq0.8>0.2\geq u_{2b}$, and after fixing $(2,a)$, $\ell_{1b}\geq0.5>0$. The other utility may remain unresolved. No claim identifies this radius with the optimal identification complexity.

\paragraph{Parallel sampling and lifetime sums.}
At the start of a group, every selected edge has width at least $w_{\max}/2$. If the oracle is unresolved, its resolution property implies $w_{\max}>\rho_h$. For a positive-noise edge with $n\geq q_h$, interval width is at most
$2\sigma\sqrt{2\Lambda/n}\leq\rho_h/\sqrt8<\rho_h/2$,
which rules out selection. When $\sigma=0$, its interval becomes a singleton after the first complete sample. An edge therefore starts a doubling batch only with $n<q_h$. Across any evidence lifetime, the last count, including a truncated issued batch, is at most twice the largest threshold in that lifetime. There are at most $1+\lceil\log_2(2q_{e,g})\rceil$ batches. Simultaneous selections use disjoint endpoints; no two chosen edges conflict in a reward-generating round.

An edge's final pull in a group precedes the final exploratory round of that group. Hence all observations arrive within the largest selected $d_e$ additional rounds, regardless of unequal batch lengths or reward-dependent arrival order. Bounding this maximum by the sum over selected edges and then grouping by lifetimes proves Eq.~\eqref{eq:parallelbudget}. It is deliberately an upper bound: simultaneous waiting need not incur the sum. An interrupted final group contributes issued pulls and actual waiting loss but does not update any incomplete estimator. If an earlier notification interrupts repair, a defined fallback remains unmatched until $T$; the stated sufficient dwell condition excludes this on the confidence event. Experimental runs also record whether each nonfinal repair actually completes.

With $n^0_{e,h}$ completed observations at an epoch boundary, an unchanged edge with $n^0_{e,h}\geq q_h$ cannot be selected. Otherwise the additional issued samples are at most $2q_h-n^0_{e,h}$. This warm-start upper bound is compatible with worse realized regret under retention: sampling decisions, batch overlap, and the old allocation's loss can differ between policies. It is not a coupling theorem establishing monotone performance.

\paragraph{Concentration and regret.}
For each selected edge the issuance indicator is predictable with respect to the latent past, even when other edges are sampled concurrently. The exponential-supermartingale argument in Appendix~\ref{app:newproofs} applies edge by edge without requiring contemporaneous independence. The estimator never uses only early arrivals. Intersections preserve truth on the common confidence event. During exploration, inserting a matching of $m$ selected edges changes at most $2m$ old assignments. Relative to the current optimal outcome, at most $a_h+2m$ requester assignments differ. The positive-part regret charge is $2a_h+4m$, and the blocking/IR charge is $(K+2N+1)(a_h+2m)$. Waiting has $m=0$. Sum over rounds, use $\sum_tm_t=P_h$, and charge confidence failure by the universal caps.

\paragraph{Historical information and repeated invalidations.}
In Theorem~\ref{thm:warm-lower}, first condition on the initial information $Z$. Couple the learner's randomization identically in the two markets. The enlarged experiment reveals a reward at issuance, then simulates its report through the common conditional kernel. Choices in the actual delayed experiment remain measurable in this enlarged experiment. Thus each newly issued pivotal reward contributes its Bernoulli KL; all other new rewards contribute zero. Any remaining historical report has identical conditional law given $Z$ and the evolving augmented transcript by assumption and contributes zero as well. The chain rule gives the stated upper bound. Data processing transfers it to the delayed transcript, and applying data processing again to the correct-output event gives the lower bound. For an unbounded stopping time with finite expected sample count, apply the chain rule to truncated transcripts and pass to the stopped experiment; infinite expected counts satisfy the bound trivially. This is an application of standard adaptive change of measure \citep{kaufmann2016complexity}.

The channel and historical-feedback conditions are essential. If a pivotal reward reports with delay zero in one market and one in the other, a single query can identify the market through arrival timing, regardless of its reward KL. Likewise, information in previously issued but still unobserved rewards must not be omitted from the information budget. Restricting the comparison to a common fixed delay and no historical pending feedback is a valid subclass of our model and suffices for the stated lower-bound constructions.

For repeated updates, fix any vector of post-update signs in the pivotal family. For each epoch compare it to a vector differing only in that epoch. The history before the epoch has identical law; unrestricted fresh means are not inferred from previous samples of the same edge. Uniform $\delta$-correctness of each repair over all sign vectors yields the one-epoch lower bound under the fixed vector. Repairs stop before the next notification, so summing their disjoint new counts gives the $L$-update statement. It does not imply an extra $\log L$ without a stronger simultaneous-error argument.

\paragraph{What the lower bound leaves open.}
For an identical pre-update history and a fixed pivotal reporting delay $d_c$, the same indistinguishability argument as Appendix~\ref{app:separations} gives expected identification time at least $(1-2\delta)(d_c+1)$. This does not show that waiting after every doubling group is necessary. The scalar resolution-radius upper bound and the one-edge information lower bound are not instance-optimal counterparts.

%% file: sections/exclusion_appendix.tex
\section{Exclusion certificates, scope, and information accounting}\label{app:exclusion}
\paragraph{Independent interval completions.}
The box in Theorem~\ref{thm:exclusion} is a Cartesian product over feasible edges. For a proposed allocation $M$, define the witness completion by $\nu_{i,M(i)}=u_{i,M(i)}$ for assigned edges and $\nu_{ij}=\ell_{ij}$ for every other edge. The outside option remains exactly zero. Different requesters' inequalities involve disjoint reward coordinates, so these choices are simultaneous. Equation~\eqref{eq:possible-stable} makes every assigned edge individually rational and every provider-willing deviation non-improving. Hence $M$ is weakly stable in this completion. Conversely, any violation of the first inequality forces individual irrationality throughout the box, and a violation of the second forces a blocking pair throughout the box. This establishes necessity and sufficiency, not merely a corner-sampling heuristic.

Let $\mathcal P(B)$ denote the set passing the test. For every $\nu\in B$, the DA allocation with the specified tie refinement is weakly stable at $\nu$ and thus lies in $\mathcal P(B)$. Consequently $|\mathcal P(B)|=1$ certifies the refined outcome. If several weakly stable allocations exist at the exact mean, all of them pass the point-box test; the exhaustive oracle need not identify the refined optimum. At zero utility, weak stability may permit an assignment whereas the specified refinement prefers the outside option. The soundness argument still holds because it only uses inclusion of the refined DA outcome in the possible set. It never equates weak stability with refined optimality.

\paragraph{Margin, robustness, and complexity.}
For any sound interval of width at most $\rho$, $\ell_e\geq\mu_e-\rho$ and $u_e\leq\mu_e+\rho$. A positive blocking witness of size at least $\beta$ stays positive by at least $\beta-2\rho$. An individual-rationality witness stays positive by at least $\beta-\rho$. Taking $\rho=\beta/3$ excludes every wrong feasible allocation while retaining the true one. The minimum in Eq.~\eqref{eq:exclusion-gap} is over a finite set; it is positive exactly when all alternatives have a positive witness. This characterizes the exhaustive weak-stability oracle's positive resolution property, not the minimal statistical complexity of refined outcome identification.

The implementable test uses only interval endpoints, feasibility and priorities. The latent margin $\beta$ is computed only in offline diagnostics. Enumeration is an explicit computational limitation: on a complete $3\times3$ market there are 34 feasible matchings including partial ones, and on a complete $5\times5$ market there are 1,546. Precomputing provider-willing pairs amortizes repeated oracle calls but does not remove exponential growth. Hybrid first calls the polynomial forced-pair oracle; if that is unresolved, it enumerates until it finds two possible allocations (unresolved) or completes the enumeration with a sole survivor. No partial enumeration can produce an unjustified certificate.

\paragraph{Checking the cyclic example.}
In Eq.~\eqref{eq:cyclic}, requester 3 cannot receive $a$ or remain unmatched: provider $b$, which ranks 3 first, would form a blocking pair. If 3 received $c$, requester 2 could not receive $c$ and would block with that provider, which prefers 2 over 3. Thus 3 receives $b$. The remaining assignment must be $1a,2c$, since $1c,2a$ is blocked by $1a$. Complete acceptability and positive utilities rule out stable partial matchings in this balanced market. Direct enumeration gives $\beta=.3$. Under $\mu_{3b}=.8+\eta$, $|\eta|\leq.01$, the same argument applies and each wrong matching has a witness at least $.29$. Forced reduction lacks a mutual first pair at the tie: the top choices of 1 and 2 are vetoed by their providers' priorities, while 3 has an unresolved choice between $b,c$.

\paragraph{Relationship to pervasive inference.}
The possibility test excludes wrong weakly stable matchings across utility completions. Pervasive inference instead asks whether the requester-optimal matching is the same for every strict refinement of partial preferences \citep{rastegari2014reasoning,athanasopoulos2026probably}. These are different questions: multiple stable allocations need not prevent a pervasive optimal allocation. We make no superiority or novelty claim for enumerating possible stable allocations. The margin supplies an explicit sufficient instance condition for the modular lifetime theorem; a stronger sound oracle can be substituted without changing that accounting argument. In our positive complete-market adapter, interval separation creates requester preference DAGs and the fixed public priority lists create provider DAGs. The published oracle then certifies a common optimum. Requiring positive lower endpoints ensures that omitting the outside option is valid. We do not apply this adapter to signed Jester or training gains.

\paragraph{A capacity-constrained information bound.}
For each alternative $\lambda$ in Eq.~\eqref{eq:history-load}, let $A$ be the event of returning $M^\star(\mu)$. Correctness gives $\Prb_\mu(A)\geq1-\delta$ and $\Prb_\lambda(A)\leq\delta$. The augmented-transcript chain rule, with a common reporting kernel and no informative unaccounted historical arrivals, gives
\[
\operatorname{kl}(1-\delta,\delta)
\leq D_0(\lambda)+\sum_e\E_\mu N_e^{\rm new}
\operatorname{kl}(\mu_e,\lambda_e).
\]
This is the usual bandit change-of-measure inequality with initial information \citep{kaufmann2016complexity}. Conditional independence of the Bernoulli edge rewards makes the simultaneous-round KL additive. For every sample path with elapsed stopping time $\tau_{\rm stop}$, feasibility implies $\sum_jN_{ij}^{\rm new}\leq\tau_{\rm stop}$ and $\sum_iN_{ij}^{\rm new}\leq\tau_{\rm stop}$. Taking expectations proves that $x_e=\E_\mu N_e^{\rm new}$ and $\tau=\E_\mu\tau_{\rm stop}$ are feasible for the optimization in Eq.~\eqref{eq:history-load}. The infimum is therefore a lower bound. Restricting to a finite subset of alternatives gives a finite linear program and preserves validity. The number of constraints or their tightness is not controlled here.

For illustration, if two distinct pivotal alternatives demand at least $b_1,b_2$ new pulls on edges sharing one requester, elapsed time is at least $b_1+b_2$. If they have disjoint endpoints, the load restriction alone requires only $\max\{b_1,b_2\}$. Historical evidence can reduce these demands differently through $D_0(\lambda)$. Delayed arrival can create additional waiting beyond the load bound; no claim makes this lower bound sufficient or matches it to the global-width upper bound.

\paragraph{Variance-sensitive evidence under an additional assumption.}
The training-replay extension uses i.i.d. reward streams within each edge lifetime. This is stronger than Assumption~\ref{ass:reward}; it is satisfied by sampling with replacement from a fixed empirical pool. Let $V_{e,n}$ be the unbiased sample variance and $\Lambda_{\rm EB}=\log(4\max\{1,|\cE|\}T_{\rm conf}^2/\delta)$. For $n\geq2$, use the empirical Bernstein radius
\begin{equation}
r_{\rm EB}(e,n)=\sqrt{2V_{e,n}\Lambda_{\rm EB}/n}
 +\frac{14\Lambda_{\rm EB}}{3(n-1)},\label{eq:eb}
\end{equation}
and retain the prior box at $n<2$. The factor 14 is the range-two rescaling of the two-sided bound of \citet{maurer2009empirical}. Applying that bound to each chronological prefix in each lifetime, and union bounding over at most $|\cE|T_{\rm conf}^2$ such pairs, gives simultaneous coverage $1-\delta$. Complete prefixes ensure that selection and reward-dependent reporting do not substitute a censored subset. Intersecting intervals within a lifetime preserves coverage. Neither this argument nor the i.i.d. assumption asserts independence between a reward and its reporting delay.

The modular lifetime proof also applies to this extension. Since $V_{e,n}\leq n/(n-1)\leq2$ for $[-1,1]$ data, replace $q_h$ by
\[
q_h^{\rm EB}=\max\left\{2,\left\lceil
512\Lambda_{\rm EB}/\rho_h^2+1+112\Lambda_{\rm EB}/(3\rho_h)
\right\rceil\right\}.
\]
For $n\geq q_h^{\rm EB}$, the square-root term is strictly below $\rho_h/8$, and the linear term is at most $\rho_h/8$. Thus width is below $\rho_h/2$, which is exactly the property used to charge doubling batches. Replace the thresholds and sufficient dwell accordingly. This worst-case bound is conservative; the empirical radius can shrink much sooner when observed variance is small. We attribute the concentration inequality to its source and do not claim a new variance estimator or a guarantee for arbitrary martingale rewards.

\paragraph{The statistical price of a prescribed trace.}
The obstruction above is not just a loose upper bound for an otherwise adequate trace learner. For $0<\eta\leq .1$, the two instances $\mu(+\eta)$ and $\mu(-\eta)$ have different first canonical proposals but the same optimal outcome. Any procedure identifying that first proposal with error at most $\delta\leq .1$ in both worlds needs
\begin{equation}
\E N_{1a}\geq\log(1/\delta)/(64\eta^2)\label{eq:trace-penalty}
\end{equation}
in either world, with no initial evidence and a shared reporting kernel: only that Bernoulli edge differs, transcript data processing requires $\operatorname{kl}(1-\delta,\delta)$ information, and its per-pull KL is at most $32\eta^2$. In contrast, sound certification of requester 2's strict preference for $a$ and requester 1's positive utility for $b$ already determines the outcome; these margins are bounded away from zero uniformly in $\eta$. Thus the oracle has a positive uniform resolution radius while the imposed trace has divergent identification cost. This compares information required by two certification objectives; it does not establish dominance over prior outcome-based methods or match expected stopping lower bounds to a finite-horizon high-probability upper bound.

\paragraph{Warm starts and capacity accounting.}
If an unchanged edge starts epoch $h$ with $n^0_{e,h}\geq q_h$ observations, its confidence width is already below the selection threshold and it needs no new pulls. If $n^0_{e,h}<q_h$, at most $2q_h-n^0_{e,h}$ new pulls can be issued. This makes the amortization operational: an update replaces the prior count by zero only where validity was lost. For a fixed realized vector of exploratory counts $P_{ij}$, any feasible matching schedule must also satisfy
\begin{equation}
 Q\geq\max\left\{\max_i\sum_jP_{ij},\ \max_j\sum_iP_{ij}\right\}\geq P/N.\label{eq:load}
\end{equation}
The first inequality follows from one assignment per participant per round, and the second from averaging requester loads. Highly concentrated information demands therefore limit parallel speedup even if many participants are available. These capacity constraints are distinct from the information needed to certify an outcome. Our greedy schedule is feasible but is not proved to attain this load bound under adaptive evidence reuse.

\section{Revision experiments and reproducibility}\label{app:revision-experiments}
\paragraph{Design and accounting.}
The exclusion suite has 270 paired tasks: 90 cyclic (three perturbations), 90 joint dynamic (three delay placements), 60 native stationary (two market sizes), and 30 signed Jester. The published PSM certificate adds 240 tasks on the same positive complete configurations. Every synthetic setting has 30 seeds. Each dataset replay adds 60 paired tasks (ten data splits, three reward seeds, two delay conditions). The binary, ten-class, EB, fresh-offer revision and annotation-revision suites are reported separately. Together with the synthetic/native suites, the revision contains 810 paired tasks, 2,250 method trajectories and 4,350 epoch records. There are 9,000 newly executed training/evaluation jobs, excluding reused empirical draws. CPU processes execute all new runs locally; connection to the previously authorized A40 server timed out, and none of these results is described as a GPU run.

Every capped run is retained. Stopping-round summaries use $\min\{\tau,T\}$ for identification policies, not a successful-run-only mean. A nonfinal failure causes unmatched fallback through the remaining horizon, with all regret and stability losses counted. Fallback epochs are never reported as completed repairs. In the cyclic joint suite, $H=4$, $T=4\times10^5$, and only edges $(1,a),(1,c)$ are invalidated. Changed-edge delay attaches to these two edges; unchanged-edge delay attaches to $(3,b),(3,c)$. A positive Bernoulli reward reports immediately; a zero reward incurs 100 rounds on a designated slow edge. Both allocations in the cycle are unique. Observed completion is empirical evidence, separate from the conservative sufficient dwell assumption.

\input{results/exclusion/joint_table}

\paragraph{Native baseline fidelity.}
PCOS-A4 implements Algorithm 4 of \citet{athanasopoulos2025pcos}, using empirical means, radii $\sqrt{\log(4NK n_{ij}^2/\delta)/(2n_{ij})}$, adaptive overlap-based active sets, and one new observation per active edge per sweep. A bipartite minimum edge coloring gives exactly the active graph's maximum degree in feasible matching rounds. It is implemented by regularizing a bipartite multigraph and repeatedly extracting perfect matchings. Native PCOS assumes complete acceptability and excludes the outside option. Our native comparison uses positive Bernoulli means and immediate feedback, within that domain. ID order resolves empirical ties. We verified the active-set and confidence update rules against the authors' public AAMAS implementation (commit \texttt{5822240f57a8701cb1d1e2b69529841dc4f1d79b}). The implementation is independent, not a claim of bitwise reproduction of the authors' experiments or random streams.

C-UCB uses empirical reward plus $\sqrt{3\log(\max\{2,t\})/(2n_{ij})}$, with unsampled edges optimistically infinite, followed by DA. C-ETC explores round-robin for $K\lceil T^{2/3}/K\rceil$ rounds, then commits to empirical DA; this budget is fixed before evaluation and not tuned on latent gaps. These two controls do not provide a $\delta$-correct stopping certificate. All native methods share seeded per-edge Bernoulli streams. Different confidence rules in full-algorithm comparisons are intentional; common-radius oracle comparisons isolate a different question.

PSM-P calls the public UAI 2026 certificate implementation at commit \texttt{f63880bedd2604839ef33f4a62e23c8791c1531b}, through \texttt{scripts/pervasive\_adapter.py}. Its source is \url{https://github.com/a-athanasopoulos/Probably-Correct-Optimal-Stable-Matching}. The wrapper constructs preference DAGs and calls the authors' super-stability/pervasiveness checks; the sampling protocol remains ours. Thus it neither reproduces full E-PSM nor attributes our delayed-retention extension to the source paper. All source hashes and configurations are recorded. Published PSM and Hybrid are both successful on the cyclic example, which precludes claiming that this family separates our outcome certificate from all earlier methods.

\paragraph{Statistical reporting.}
Synthetic/native summaries report mean and SE over 30 paired seeds or markets. The joint-regret intervals use 10,000 paired bootstrap resamples, seed 20270914, with no multiplicity correction; they are descriptive. We report the changed-edge Hybrid interval even though it includes zero. For training replay, three policy seeds share each data split: uncertainty is computed from ten split-level averages, not from treating 30 seeds as independent datasets. Exact completion counts are also reported. Data splits reuse observations from the same finite public dataset, so split variation does not imply independent population samples.

\paragraph{Real training-feedback distributions.}
Both protocols use \texttt{sklearn.datasets.load\_digits}, 1,797 images, features divided by 16, stratified 70/30 splits with seeds $31000+s$, $s=0,\ldots,9$. Candidate offers are nested prefixes of one randomized training pool of sizes 20/40/80/120. Thirty independently seeded bootstrap trials per buyer--offer pair produce 3,600 actual training/evaluation jobs per protocol. No held-out scores select offers or model parameters. Every job starts from zero and uses 150 full-batch gradient steps, learning rate .2 and weight penalty .001. Split indices, features, anchor examples, offers, job rewards, and wall times are retained.

The binary protocol uses tasks $(0,1),(3,8),(4,9)$, with five anchor examples per class. It uses utility $\operatorname{clip}(\text{accuracy gain}/.1,-1,1)$. The follow-up ten-class protocol uses three fixed feature views (all 64 pixels, the central $6\times6$ pixels, and every other flattened pixel), two anchor examples per class, and multinomial logistic regression. Its utility is $\operatorname{clip}(\text{accuracy gain}/.5,-1,1)$. The constants define the reported utility scale and are not chosen per market. Binary failure motivated the second protocol; both complete result sets remain reported, without claiming a preregistered study.

Replay draws a stored reward--duration pair uniformly with replacement whenever the policy issues an edge pull. It therefore preserves dependence between training gain and duration within each empirical pool. Its target mean is exactly that pool's average, which is used only for offline matching and loss evaluation. A 1 ms decision slot converts a measured job duration $s$ into delay $\lceil s/.001\rceil-1$, matching the simulator's convention that zero-delay feedback is available at the next round. An immediate-feedback control uses the same reward sequence. Capacity and provider priorities (seed $6161+s$ per split) are simulated. The empirical maximum delay is a bound for this replay only, not for future training jobs. Rewards are in $[-1,1]$, and the initial replay uses the conservative $\sigma=1$ radius and $T=10^5$. A separate ten-class EB ablation uses Eq.~\eqref{eq:eb}, the same feedback pools and seeds, and the same $\delta,T$; it performs no new training jobs. The variance ablation was added after diagnosing conservative-radius failures. We retain both versions and do not attribute their difference solely to the outcome oracle.

\input{results/exclusion/replay_table}

\paragraph{Notified revisions with actual training feedback.}
We extend the ten-class empirical market to two epochs of $300{,}000$ rounds. In the first revision condition, provider 1's 20-example offer is replaced by a fresh training subset, excluding its original offer and the anchor/validation examples. Features, the anchor, and the validation split remain fixed. Three buyers, ten splits, and 30 trials require 900 additional actual jobs; all other feedback pools are identical across epochs. The notification invalidates precisely this provider's three edges. All policies use EB intervals and are compared under retaining unchanged-edge statistics versus resetting every edge. Initial failures trigger the full-horizon fallback. The fresh-offer condition changes the target allocation in only one of ten splits, so its main interpretation is outcome revalidation.

\input{results/exclusion/training_repair_table}

A second, separately reported condition changes the annotations of the fixed 120-example offer of provider 4. Exactly 24 training labels are replaced by a label uniformly selected from the other nine classes, with corruption held fixed across bootstrap trials. This controlled 20\% symmetric-noise stress follows the type of label corruption used by \citet{han2018coteaching}; our digits/softmax market is an adaptation, not their benchmark or algorithm. Anchor and validation labels remain clean, and the other providers' pools are unchanged. This adds another 900 actual jobs and changes the target allocation in five of ten data splits. Provider identity, noise rate, seeds, and all cases are fixed before running these policies; neither condition selects markets by certification success. The experiment tests a specified empirical quality shift, not naturally occurring commercial data deterioration.

Table~\ref{tab:training-quality} reports both delay conditions, including all failures.

\paragraph{Verification and limits.}
Tests check possible-stability equivalence against all endpoint completions on small boxes, all 512 three-by-three graphs for minimum edge coloring, signed/outside edge cases, cyclic perturbations, and capped baseline runs. Existing complete-batch and feasible-schedule tests remain in use. All 36 implementation tests pass, including a separate check of the published PSM adapter. A trajectory audit checks 157,839 batch groups, observes no incorrect output among 2,583 completed epochs, and verifies the pathwise loss bounds in all 2,565 completed batch-protocol epochs. The other 18 completed epochs belong to PCOS-A4. The audit retains 604 fallback epoch records. These tests establish finite implementation checks; the proofs above remain necessary. In particular, neither a finite set of correctly completed runs nor a passed test suite establishes statistical coverage for every instance.

%% file: results/exclusion/joint_table.tex
\begin{table}[htbp]
\centering\small
\caption{Four-epoch cyclic repairs. Paired total-regret difference (retain minus reset), with descriptive 95\% bootstrap intervals over 30 seeds. All Hybrid and PSM repairs complete.}
\label{tab:joint-new}
\begin{tabular}{lrr}\toprule
Delayed edges&Hybrid&Published PSM oracle\\\midrule
None & -787 [-1,412, -178] & -1,108 [-1,742, -494]\\
Changed & -559 [-1,156, 24] & -880 [-1,495, -284]\\
Unchanged & -1,897 [-2,509, -1,293] & -2,226 [-2,853, -1,612]\\
\bottomrule\end{tabular}
\end{table}

%% file: results/exclusion/replay_table.tex
\begin{table}[htbp]
\centering\small
\caption{Empirical training-feedback replay: certified trials out of 30 and mean total regret $\pm$ SE across ten split-level averages (three reward seeds each). The two task protocols use different utility scales; regret magnitudes should not be compared across protocols.}
\label{tab:replay-new}
\begin{tabular}{llrrrr}\toprule
Protocol&Method&\multicolumn{2}{c}{Immediate}&\multicolumn{2}{c}{Measured delay}\\
&&Cert.&$\mathcal R_T$&Cert.&$\mathcal R_T$\\\midrule
Binary & Forced & 0/30 & 25,321 $\pm$ 3,432 & 0/30 & 25,430 $\pm$ 3,464\\
Binary & Hybrid & 0/30 & 25,321 $\pm$ 3,432 & 0/30 & 25,430 $\pm$ 3,464\\
Binary & Prefix & 0/30 & 25,321 $\pm$ 3,432 & 0/30 & 25,430 $\pm$ 3,464\\
Ten-class & Forced & 0/30 & 26,680 $\pm$ 1,780 & 0/30 & 26,704 $\pm$ 1,782\\
Ten-class & Hybrid & 2/30 & 26,537 $\pm$ 1,807 & 2/30 & 26,564 $\pm$ 1,809\\
Ten-class & Prefix & 0/30 & 26,680 $\pm$ 1,780 & 0/30 & 26,704 $\pm$ 1,782\\
Ten-class EB & Forced & 21/30 & 15,627 $\pm$ 1,850 & 21/30 & 15,691 $\pm$ 1,849\\
Ten-class EB & Hybrid & 21/30 & 14,064 $\pm$ 2,264 & 21/30 & 14,127 $\pm$ 2,262\\
Ten-class EB & Prefix & 12/30 & 20,072 $\pm$ 2,269 & 12/30 & 20,120 $\pm$ 2,269\\
\bottomrule\end{tabular}
\end{table}

%% file: results/exclusion/training_repair_table.tex
\begin{table}[htbp]
\centering\small
\caption{Empirical provider revision with EB intervals, two epochs of $300{,}000$ rounds. Each difference is retain minus reset, averaged within each of ten splits before constructing a paired 95\% bootstrap interval. Both-epoch completion counts include every run and failure.}
\label{tab:training-repair}
\begin{tabular}{llrrr}\toprule
Delay&Oracle&Retain complete&Reset complete&Regret difference [95\% CI]\\\midrule
Immediate & Forced & 23/30 & 23/30 & -8,990 [-13,944, -3,988]\\
Immediate & Hybrid & 23/30 & 23/30 & -7,605 [-12,845, -2,670]\\
Measured & Forced & 23/30 & 23/30 & -8,989 [-13,944, -3,981]\\
Measured & Hybrid & 23/30 & 23/30 & -7,604 [-12,845, -2,665]\\
\bottomrule\end{tabular}
\end{table}

%% file: sections/load_experiment_details.tex
\section{Residual-load experiment protocol}\label{app:loadexperiments}
All new policy runs use $\delta=.01$, chronological complete-phase evidence, and the same finite horizon in each paired comparison. The implementation uses a compressed exact multigraph decomposition. Every matching block launches its prescribed jobs before the next block; it does not wait for prior blocks within that phase. Feedback is committed only after every phase reward arrives. Comparison names distinguish RLR (exact packing, one phase barrier), greedy packing (same phase quotas and barrier), block barriers (same exact packing, a wait after every compressed block), and the preceding width-greedy O-CGR. The first three have identical completed-phase prefixes for a fixed number of completed phases; horizon truncation can produce different issued counts. Their allocations within a phase can still have different regret even with identical completion time.

\paragraph{Larger markets.}
For each size $N=K\in\{8,16,32\}$ and each of ten seeds, draw provider priorities independently uniformly over permutations. Unfiltered means are independently uniform on $[.1,.9]$. A separate, labeled separated condition sets diagonal means to $.85$ and off-diagonal means independently uniform on $[.1,.45]$. The separated condition diagnoses execution cost without small optimal-assignment gaps; it is not a replacement for the unfiltered condition. Use independent Bernoulli streams per edge, a $30{,}000$-round horizon, and constant delay zero or 100. Each seed is paired across all four scheduling rules with Forced-P fixed. Market seeds are $277001+s$; feedback streams use the same edge/epoch seeding scheme as the preceding suite. All failures are charged through the horizon. No result is filtered by realized gaps or certification success.

\paragraph{Published stopping rule.}
Thirty additional random markets each at sizes 3 and 5 compare RLR and width-greedy scheduling with the same pinned author PSM oracle, radius, and streams. The means and priorities use the same generation rule as the larger-market suite. These are new paired seeds relative to the preceding PCOS table, so cross-table completion differences are not causal method comparisons. PSM-P remains a certificate adaptation, not a full E-PSM implementation. The original complete PCOS-A4 and regret-policy comparison is retained in Appendix~\ref{app:preload}.

\paragraph{Empirical training revisions.}
Reuse exactly the previously measured ten-class quality-shift pools: ten splits, three policy seeds per split, 20\% fixed corruption of the designated large offer, two epochs of $300{,}000$ rounds, and measured or immediate delays. No model is retrained in this revision. Cross two stopping rules (Forced, Hybrid), two schedules (RLR, width-greedy), and retention/reset. Means are available only to the evaluator; the policy receives selected reward--duration pairs sampled with replacement. The original empirical Bernstein intervals and all data preprocessing stay fixed. A failed first epoch triggers the documented unmatched fallback for the entire second epoch. Paired bootstrap intervals average the three policy seeds within each split before resampling the ten splits, avoiding treating seeds as independent datasets.

\paragraph{Frozen-quota diagnosis.}
At sizes 8, 16, 32, and 64, thirty seeded integer target matrices draw each entry uniformly from 100 through 1000. A row update invalidates all edges of one requester; a dispersed update invalidates the diagonal. Both invalidate $N$ edges. Unchanged entries retain their full target count; reset discards all entries. We measure the exact decomposition length, endpoint lower bound, total pulls, and CPU decomposition time. These are prescribed-demand scheduling measurements, not stochastic identification or new model-training trials. The public quota input is not a latent gap supplied to RLR.

\paragraph{Audit.}
Raw outputs include every task, method, failure, phase quota, and code snapshot. Checks recompute endpoint loads, phase issue totals, horizon accounting, correctness against evaluator DA on completed runs, and identical-prefix scheduling ablations where both finish. Small random multigraph tests independently verify exact quota realization and matching feasibility, including unequal counts and dummy-edge removal. A separate residual-count test checks the factor-two dyadic inequality over heterogeneous initial counts. These implementation checks are not independent mathematical peer review.

\subsection{Width-adaptive follow-up and limits of the guarantee}
After observing the dyadic policy's negative empirical replay result, we evaluated a declared follow-up, RLR-W, without changing the data pools, horizon, confidence level, or stopping rules. At each phase, retain every edge with width at least half the maximum width and request $\max(1,n_e)$ new samples on each. Use the exact residual-demand decomposition and commit after a single complete-phase barrier. This is the previous width eligibility and doubling rule applied to a whole feasible sequence of matchings, rather than a single matching group. It is not the dyadic algorithm of Theorem~\ref{thm:load}.

For phase demand $x^{(s)}$, the exploratory time is exactly $\mathsf L(x^{(s)})$ and waiting is at most $d$. Thus $Q=\sum_s\mathsf L(x^{(s)})$ and $W\leq dS$ for $S$ nonempty phases, with the same insertion-loss bound. If unresolved on the confidence event, every selected edge has width greater than $\rho_h/2$. The argument of Theorem~\ref{thm:modular} bounds its start count below $q_h$, hence $P\leq2\sum_{e,g}q_{e,g}$. Charging each phase to a selected edge yields the conservative $S\leq\sum_{e,g}b_{e,g}$. We do not transfer the dyadic factor-two residual-load bound or its logarithmic per-epoch phase count to this variant.

In measured-delay Hybrid replay, RLR-W retain/reset regret is 45,826/53,928, with 20/30 complete runs for each. Its retain-minus-reset difference is $-8,102$, split-bootstrap 95\% interval $[-13,957,-3,040]$. However, RLR-W retain minus Width retain is $+5,280$ with interval $[2,686,8,264]$. Dyadic RLR retain minus Width retain is also positive: $+3,999$ with interval $[1,412,7,220]$. The dyadic retain-minus-reset difference is $-9,384$ with interval $[-14,824,-4,589]$. Thus history reuse helps within each protocol, but neither new phase policy improves the measured-delay replay relative to Width. The intervals are descriptive and not corrected for multiple comparisons. This is an exploratory reuse of the same ten splits, not independent validation after method selection.

\subsection{Known-hypothesis likelihood benchmark}
The separate simulator \texttt{scripts/fallback\_identification.py} supplies public candidate gaps $\Delta_{ij}=.1+.01((i+2j)\bmod5)$ on a $4\times4$ candidate graph, with private known fallback providers. Only the observation generator knows the actual sign. For the reference-market plots, all candidate means are $1/2-\Delta_{ij}$. The policy uses likelihood boundaries $A=\log(2|E_c|/\delta)$ and targets $\lceil(A+A^{3/4})/\kappa_e\rceil$, doubled on subsequent unresolved phases. Initial sample counts are deterministic $\lfloor\alpha\log(1/\delta)/\kappa_e\rfloor$. Initial sums and subsequent phase sums use independent binomial draws, exactly equivalent to summing independent Bernoulli observations. Since the likelihood depends only on their sum, aggregation changes neither the policy nor its distribution.

The seven log-inverse-error values and four history fractions produce 28 configurations with 100 seeds each. The experiment uses immediate feedback and reports expected-time estimates normalized by $\log(1/\delta)$, together with the analytic first-order load, not a success-conditioned time. Very small $\delta$ values are represented by their logarithm to avoid numerical underflow; the likelihood boundary does not require materializing $\delta$. Numerical evaluations of the finite lower bound use its limiting binary-divergence expression only when the correction is below floating-point precision. All 2,800 reference runs finish without an observed incorrect sign; these simulations cannot verify nominal error probabilities such as $\exp(-3000)$. Separate implementation tests cover each single-positive alternative and compare its unique stable outcome by exhaustive enumeration on small graphs. Finite-confidence overhead is visible and is not removed from the plot.

%% file: sections/pre_load_experiments.tex
\section{Previous certificate and reuse experiments}\label{app:preload}
We compare certificate acceptance, end-to-end learning, and reuse after notification. Forced-P is the original polynomial outcome oracle; Hybrid-P adds Theorem~\ref{thm:exclusion}. PSM-P calls the UAI 2026 authors' pervasive-matching implementation with public provider priorities, restricting it to certified-positive complete markets. These three use the same horizon-dependent confidence radius, $\delta=.01$, parallel width schedule, and paired per-edge reward streams. PSM-P is a certificate comparison, not full E-PSM. Separately, PCOS-A4 implements the complete adaptive sampler of \citet{athanasopoulos2025pcos}, including its native per-edge radius, active-pair updates, and minimum edge-coloring sweeps. C-UCB and C-ETC provide regret controls \citep{liu2020competing}. Appendix~\ref{app:revision-experiments} gives source commits, parameter choices, uncertainty estimates, and all failures.

\paragraph{A structural failure beyond forced pairs.}
For Eq.~\eqref{eq:cyclic} with $\eta\in\{0,.001,.01\}$, Bernoulli feedback, and $T=10^5$, each configuration uses 30 seeds. At the exact tie, Forced-P completes 0/30 trials; Hybrid-P and PSM-P both complete 30/30. Hybrid-P uses about 3,680 rounds. The outcome remains unchanged under the perturbations. Enumeration of all 34 feasible allocations independently verifies uniqueness; this is a diagnostic beyond observing correct outputs. The benefit is shared with the existing pervasive oracle, as the theory comparison predicts.

\begin{figure}[htbp]
\centering\includegraphics[width=.96\linewidth]{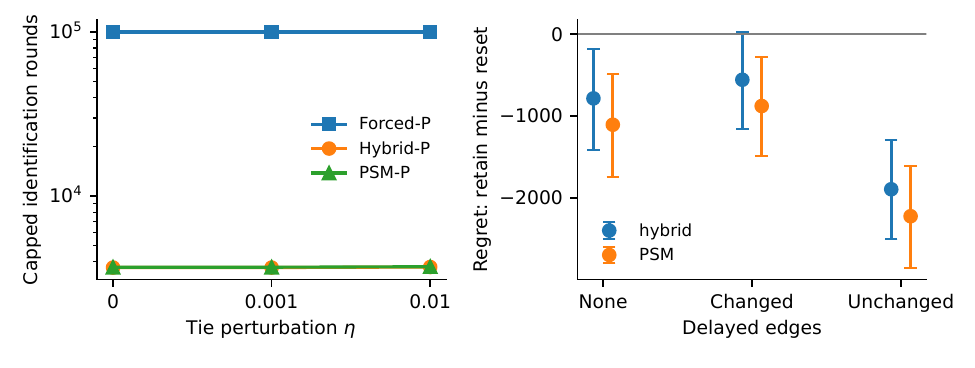}
\caption{Left: the cyclic family, with unsuccessful runs capped at $10^5$ rounds. Right: paired total-regret differences across four epochs; bars are descriptive 95\% bootstrap intervals over 30 seeds. Negative values favor retention. The published PSM oracle benefits from the same reuse protocol.}\label{fig:revision}
\end{figure}

\paragraph{Joint certification and reuse.}
We alternate Eq.~\eqref{eq:cyclic} with a profile changing only requester 1's utilities for $a,c$ from $(.8,.5)$ to $(.5,.8)$. The unique outcome changes to $(1c,2a,3b)$. Four epochs each last $10^5$ rounds; notifications name precisely these two edges. We compare retain/reset under immediate feedback and under reward-dependent delays up to 100 on changed or unchanged edges. Every Hybrid/PSM repair completes; Forced-P fails in the initial tied epoch and then follows the specified unmatched fallback, whose full-horizon loss is counted. These finite-budget trials demonstrate empirical completion; their dwell is below the conservative sufficient theorem bound. Figure~\ref{fig:revision} isolates reuse within each successful certificate; Appendix~\ref{app:revision-experiments} gives the numerical intervals. In contrast, the earlier ring suite gives identical retained trajectories for all three original certificates (Appendix~\ref{app:newexperiments}), so its reuse benefit alone could not establish a new certificate advantage.

\paragraph{Native full-algorithm comparison.}
For $N=K\in\{3,5\}$, independently draw utilities uniformly from $[.1,.9]$ and provider priorities uniformly over permutations, with 30 markets per size and $T=30{,}000$. No gap filtering or favorable-market selection is applied. Table~\ref{tab:native-new} reports all trials. At size three, Hybrid-P certifies 20/30, PSM-P 21/30, and PCOS-A4 17/30; size five remains difficult for all certification policies. Low regret from C-UCB/C-ETC does not imply certified identification, but the large regret gap shows that these repair protocols are not competitive stationary regret policies. We report that distinction explicitly rather than infer general superiority from the cyclic construction.

\input{results/exclusion/native_table}

\paragraph{Real-data failures and training-feedback replay.}
On the unchanged signed-Jester protocol, Hybrid-P completes 5/30 trials versus Forced-P's 4/30. At exact-mean boxes, both oracles certify 28/30 markets: most observed failures are therefore finite-budget failures, while two remain structural for these oracles. Training replay uses ten stratified data splits, nested offers of 20/40/80/120 examples, and 30 actual training trials per buyer--offer pair. Policies sample the empirical joint distribution of validation gain and measured duration with replacement; durations define delays in 1 ms simulation slots. Means are known only to the evaluator. This is an empirical replay with simulated market rules. Binary tasks yield no completed certifications; ten-class feature views give Hybrid 2/30 with the conservative radius. An i.i.d.-stream extension using empirical Bernstein intervals \citep{maurer2009empirical} raises ten-class completion to 21/30 for both Hybrid and Forced, versus 12/30 for Prefix. This improvement is mainly due to variance adaptation, not a unique certificate advantage. Appendix~\ref{app:revision-experiments} reports split-level uncertainty, both failed protocols, and the additional assumption and proof.

\paragraph{Revalidation after an actual offer revision.}
Two additional empirical replay conditions use newly trained feedback: replacing a small offer, and corrupting 20\% of the labels in a fixed large offer, following the symmetric-noise stress used in noisy-label learning \citep{han2018coteaching}. The latter changes the target allocation in 5/10 splits. With measured delays, two epochs of $300{,}000$ rounds, and EB intervals, Hybrid-retain has mean regret 40,545 versus 46,820 after reset; the paired difference is $-6{,}275$ with split-level 95\% bootstrap interval $[-9{,}374,-3{,}463]$. Both complete 20/30 double-epoch runs. All failures and the less disruptive fresh-offer condition remain in Appendix~\ref{app:revision-experiments}; these are controlled empirical revisions, not commercial deployment results.

\input{results/exclusion/training_repair_quality_table}

\FloatBarrier

%% file: results/exclusion/native_table.tex
\begin{table}[htbp]
\centering\small
\caption{Native stationary comparison, $T=30{,}000$. Mean $\pm$ SE over 30 independent markets per size; failed identification is capped at $T$. C-UCB/C-ETC do not certify. PSM-P is an oracle ablation; PCOS-A4 is a full sampling baseline.}
\label{tab:native-new}
\begin{tabular}{llrrr}\toprule
Size&Method&Certified&Capped rounds&$\mathcal R_T$\\\midrule
$3\times3$ & Forced-P & 19/30 & 19,607 $\pm$ 1,881 & 16,423 $\pm$ 1,902\\
 & Hybrid-P & 20/30 & 18,838 $\pm$ 1,893 & 15,918 $\pm$ 1,936\\
 & PSM-P & 21/30 & 18,590 $\pm$ 1,764 & 15,658 $\pm$ 1,773\\
 & PCOS-A4 & 17/30 & 19,585 $\pm$ 1,926 & 19,031 $\pm$ 2,348\\
 & C-UCB & -- & -- & 585 $\pm$ 141\\
 & C-ETC & -- & -- & 801 $\pm$ 256\\
\midrule
$5\times5$ & Forced-P & 0/30 & 30,000 $\pm$ 0 & 57,848 $\pm$ 2,213\\
 & Hybrid-P & 1/30 & 29,762 $\pm$ 238 & 57,560 $\pm$ 2,332\\
 & PSM-P & 3/30 & 29,126 $\pm$ 518 & 56,171 $\pm$ 2,381\\
 & PCOS-A4 & 1/30 & 29,884 $\pm$ 116 & 50,779 $\pm$ 2,419\\
 & C-UCB & -- & -- & 2,037 $\pm$ 371\\
 & C-ETC & -- & -- & 3,209 $\pm$ 1,035\\
\bottomrule\end{tabular}
\end{table}

%% file: results/exclusion/training_repair_quality_table.tex
\begin{table}[htbp]
\centering\small
\caption{Controlled label-quality revision with EB intervals, two epochs of $300{,}000$ rounds. Each difference is retain minus reset, averaged within each of ten splits before constructing a paired 95\% bootstrap interval. Both-epoch completion counts include every run and failure.}
\label{tab:training-quality}
\begin{tabular}{llrrr}\toprule
Delay&Oracle&Retain complete&Reset complete&Regret difference [95\% CI]\\\midrule
Immediate & Forced & 20/30 & 20/30 & -6,190 [-9,061, -3,538]\\
Immediate & Hybrid & 20/30 & 20/30 & -6,283 [-9,386, -3,462]\\
Measured & Forced & 20/30 & 20/30 & -6,188 [-9,055, -3,541]\\
Measured & Hybrid & 20/30 & 20/30 & -6,275 [-9,374, -3,463]\\
\bottomrule\end{tabular}
\end{table}

%% file: sections/trace_reference.tex
\section{A certificate for the proposal trace}\label{sec:trace}
Fix canonical DA: always process the smallest-ID unmatched requester that has not exhausted its options. At a query $q$, requester $i_q$ chooses its best element from its untried providers $A_q\setminus\{0\}$ and outside option $0$. A proposal deletes that provider from its untried set. Providers retain their preferred requester among new proposals and their current tentative partner. If the outside option wins, the requester becomes exhausted. There are at most $|\cE|+N$ queries. Each real proposal is made at most once.

Suppose each edge has an interval $[\ell_{i,j},u_{i,j}]$, and $\ell_{i,0}=u_{i,0}=0$. A query is certified when an alternative $a_q$ satisfies
\begin{equation}
 \ell_{i_q,a_q}>\max_{b\in A_q\setminus\{a_q\}}u_{i_q,b}.\label{eq:querycert}
\end{equation}
A sole outside option is certified without a comparison. If no alternative satisfies Eq.~\eqref{eq:querycert}, the procedure returns the first unresolved query. It never fills an unresolved preference arbitrarily.

\begin{theorem}[Trace certificate]\label{thm:trace}
If every query in a completed canonical trace is certified, every utility matrix inside the interval box has the same canonical trace and requester-optimal DA outcome. A full check costs $O((|\cE|+N)(K+1))$ comparisons. No ordering between losing alternatives is required.
\end{theorem}
\begin{proof}
Induct over queries. The prior proposal state is identical for all matrices in the box. Equation~\eqref{eq:querycert} fixes the next choice, and public provider priorities fix its acceptance and rejection. Thus the next state is identical. Termination and requester optimality follow from DA (Appendix~\ref{app:da}). Each query checks at most $K+1$ alternatives.
\end{proof}
This is a sufficient certificate for optimality, not a characterization of all interval boxes with a common optimal outcome. A box may have a common outcome but different proposal traces. In particular, we do not solve minimum-certificate discovery.

For epoch $h$, let $\mathcal T_h$ be its true canonical trace. For the original trace-specific guarantees, we assume each nontrivial query has a unique best alternative, while permitting arbitrary ties among losers. For real edge $e=(i,j)$, define its \emph{trace margin}
\begin{equation}
 \Gamma_{e,h}=\min_{\substack{q\in\mathcal T_h,\ i_q=i,\ b\in A_q\setminus\{a_q\}\\j\in\{a_q,b\}}}
       (\mu_{i,a_q,h}-\mu_{i,b,h}),\label{eq:tracegap}
\end{equation}
with $\min\varnothing=+\infty$. Only winner--competitor comparisons enter. These are analysis parameters; the learner is not given them.

\begin{proposition}[Separation from the global ranking gap]\label{prop:gap-separation}
With one requester and means $(0.8,0.2,0.2+\eta)$, $0\leq\eta\leq0.1$, every finite trace margin is at least $0.5$, whereas the minimum pairwise gap is zero at $\eta=0$ and at most $\eta$ otherwise. Thus the bounds below remain finite as $\eta\downarrow0$.
\end{proposition}
The only query selects provider 1 over the two losing providers and the outside option. This example separates our bound from global-gap analysis; adaptive partial-ranking predecessors may also avoid the losing comparison.

\section{Certificate-Guided Repair}\label{sec:cgr}
CGR stores a sample count, sum, and confidence interval for each edge. At a notification, it resets only edges in $\mathcal C_h$. It replays the canonical trace from the empty DA state, retaining all valid statistical evidence. If a prefix is still certified, its proposal states can be reused; the reference implementation rechecks the trace for simplicity. This is statistical reuse, not a claimed sublinear-time dynamic matching algorithm.

For an unresolved query with alternative set $A$, form its possible winners
\begin{equation}
 \mathcal P=\{a\in A:u_{i,a}\geq\max_{b\in A}\ell_{i,b}\}.\label{eq:possible}
\end{equation}
Choose a real alternative in $\mathcal P$ with the largest interval width, breaking ties by provider ID. If it has $n$ completed samples, issue enough pulls to reach $n'=\max\{1,2n\}$. Wait until \emph{all} rewards in this batch have arrived, update the interval, and recheck the certificate. No other batch is pending. This serialization is conservative but makes the delay and sampling arguments explicit. An interrupted nonfinal repair triggers the fallback specified in Appendix~\ref{app:newproofs}.

Let $M^{\rm old}$ be the preceding epoch's learned matching, initially empty. During a pull of $(i,j)$, remove $i$'s old assignment and the old occupant of $j$, assign $i$ to $j$, and leave every other old pair intact. During waiting, execute $M^{\rm old}$. Once certification completes, execute its outcome for the rest of the epoch. All exploration rewards and all waiting-round losses count. Rewards from other executed pairs are ignored by this estimator.

Use the confidence radius in Eq.~\eqref{eq:cs}.

After a completed batch, intersect the previous interval and the prior with $[\widehat\mu_e-r(n),\widehat\mu_e+r(n)]\cap[-1,1]$. With no samples use the prior. If the intersection is empty, a defined fallback uses the latest clipped interval; this occurs only off the confidence event. This rule does not condition the estimator on which rewards returned early.

\begin{algorithm}[t]
\caption{Certificate-Guided Repair (CGR)}\label{alg:cgr}
\KwIn{public graph/priorities, $T,\sigma,\delta$; notified edges $\mathcal C_h$}
Invalidate only $\mathcal C_h$; retain all other edge evidence\;
\While{canonical DA is not fully interval-certified}{
 Let $(i,A)$ be its first unresolved query\;
 Choose the widest real possible winner using Eq.~\eqref{eq:possible}\;
 Double its sample count (start at one), executing modified old matchings\;
 Wait for its whole batch, executing the old matching\;
 Update its confidence interval; replay certified DA\;
}
Execute the certified outcome until the next notification\;
\end{algorithm}

\section{Theoretical guarantees}\label{sec:newanalysis}
\subsection{Confidence and selective sampling}
\begin{lemma}[Uniform complete-batch confidence]\label{lem:adaptive-confidence}
Under Assumptions~\ref{ass:reward}--\ref{ass:announced}, with probability at least $1-\delta$, every interval used by CGR contains its edge's current mean. The claim holds for adaptive edge selection and reward-dependent bounded delays.
\end{lemma}
The proof uses chronological samples within an evidence lifetime, a stopped exponential supermartingale, and a union bound over at most $|\cE|T^2$ edge--lifetime--count triples. Each completed batch contributes \emph{all} its samples; arrival ordering cannot select a favorable subset. Full details are in Appendix~\ref{app:newproofs}.

Define, for each finite trace margin,
\begin{equation}
 n^\star_{e,h}=\max\left\{1,\left\lceil\frac{64\sigma^2\Lambda}{\Gamma_{e,h}^2}\right\rceil\right\},\qquad
 b^\star_{e,h}=1+\left\lceil\log_2(2n^\star_{e,h})\right\rceil.\label{eq:threshold}
\end{equation}
An edge with no trace comparison is omitted from all sums. Over evidence lifetime $g$ of edge $e$, let $n^\star_{e,g}$ be the maximum of $n^\star_{e,h}$ over its epochs, and define $b^\star_{e,g}$ analogously.

\begin{theorem}[Selective sample and waiting budgets]\label{thm:selective}
On the event of Lemma~\ref{lem:adaptive-confidence}, CGR never starts a batch on edge $e$ in epoch $h$ when its completed count is at least $n^\star_{e,h}$. If $P_{e,g}$ is the number of issued exploratory pulls in lifetime $(e,g)$ and $B_{e,g}$ its batch count, then
\begin{equation}
 P_{e,g}\leq2n^\star_{e,g},\qquad B_{e,g}\leq b^\star_{e,g}.\label{eq:lifetime}
\end{equation}
Let $\mathcal J$ be the set of edge lifetimes actually sampled. Total exploratory pulls $P$ and waiting rounds $W$ satisfy
\begin{equation}
 P\leq2\sum_{(e,g)\in\mathcal J}n^\star_{e,g},\qquad
 W\leq\sum_{(e,g)\in\mathcal J}d_e b^\star_{e,g}.\label{eq:selectivebudgets}
\end{equation}
Replacing $\mathcal J$ by all edge lifetimes with finite trace margins gives deterministic upper bounds. In epoch $h$, a sufficient repair duration is
\begin{equation}
 F_h=\sum_{e:\Gamma_{e,h}<\infty}(2n^\star_{e,h}+d_e b^\star_{e,h}).\label{eq:dwell}
\end{equation}
Thus repair completes before the next notification when the nonfinal duration is at least this $F_h$, on the confidence event. The final epoch may be truncated by $T$.
\end{theorem}
\paragraph{Proof idea.}
On the confidence event, every completed query matches the true trace. Its true winner remains in $\mathcal P$. A sampled widest candidate has width at least half its applicable winner--competitor gap; otherwise the winner would already be certified. Since interval width is at most $2r(n)$, an edge with $n\geq n^\star_{e,h}$ cannot be selected. Doubling overshoots this threshold by at most a factor two. Counts persist across other edges' updates, so the bound is a maximum over a lifetime, not a sum over market epochs. Each complete batch waits at most $d_e$ additional rounds. Appendix~\ref{app:newproofs} supplies the candidate-width argument, including the outside option.

\subsection{Regret and stability during repair}
\begin{theorem}[Cost of announced repair]\label{thm:local-repair}
Let $a_h=|\{i:M_{h-1}^\star(i)\ne M_h^\star(i)\}|$, with $M_0^\star$ empty, and let $P_h,W_h$ be CGR's exploratory and waiting rounds in epoch $h$. On the confidence event, all completed repairs are requester-optimal, and the realized positive-part regret $\mathcal R_T$ and exact violation count $\mathcal S_T^0$ obey
\begin{align}
 \mathcal R_T&\leq\sum_h\{2a_h(P_h+W_h)+4P_h\},\label{eq:localregret}\\
 \mathcal S_T^0&\leq(K+2N+1)\sum_h\{a_h(P_h+W_h)+2P_h\}.\label{eq:localstability}
\end{align}
Taking expectations adds at most $2NT\delta$ and $N(K+1)T\delta$, respectively. For deterministic budgets $\overline P,\overline W$ from Theorem~\ref{thm:selective}, a direct consequence is
\begin{equation}
 R_T\leq \min\{2NT,(2a_{\max}+4)\overline P+2a_{\max}\overline W+2NT\delta\},\label{eq:detregret}
\end{equation}
where $a_{\max}=\max_h a_h$. None of $a_h,\Gamma_{e,h}$ or $F_h$ is an input to the sampler.
\end{theorem}
\begin{proof}
During waiting, the old matching differs from the current optimal one on $a_h$ requesters. During exploration at most two more assignments differ. The positive-part triangle inequality charges at most $2a_h$ for the old allocation and 4 for the exploratory modification. For stability, changing $q$ requester assignments changes at most $2q$ provider partners, creating at most $(K+2N)q$ blocking pairs and $q$ individual-rationality violations. Certified commitment has zero loss. Sum these bounds; charge the confidence failure by the universal caps.
\end{proof}

\paragraph{Interpretation of the bounds.}
Statistical locality is measured by edge lifetimes and the comparisons that require new samples. Allocation locality is measured by $a_h$. They can differ dramatically. Waiting is charged to queried edges, not automatically to the slowest edge in the market; nevertheless CGR can still wait on every selected batch and is not delay-optimal. With few invalidations, the lifetime budget can be much smaller than relearning all edges every epoch.

\section{Separations and lower bounds}\label{sec:separations}
\begin{proposition}[Sparse update, global reassignment]\label{prop:cascade}
For every $N\geq2$, a connected market with two feasible edges per requester admits a change in one requester's utilities that changes all $N$ partners in the requester-optimal outcome.
\end{proposition}
Number participants modulo $N$. Requester $i$ has edges to $i$ and $i+1$, with means $0.8$ and $0.3$. Provider $j$ ranks requester $j-1$ above $j$. Initially DA matches every $i$ to $i$. Swap the two means of requester 0. DA now gives $i$ provider $i+1$ for every $i$, by a rejection chain. Only two edges changed. With sufficiently precise retained intervals, only those two edges need fresh evidence even though $a_h=N$. Appendix~\ref{app:separations} verifies both statements.

\begin{proposition}[Information and delay on a pivotal edge]\label{prop:pivotal}
Consider the two markets with means
\begin{equation}
 \mu^{\pm}=\begin{pmatrix}1/2\pm\eta&1/2\\0.9&0.1\end{pmatrix},\qquad
 1\succ_a2,\quad 2\succ_b1,\quad 0<\eta\leq1/4.\label{eq:pivotal}
\end{equation}
Rewards are independent Bernoulli, and only edge $(1,a)$ differs. Any algorithm identifying the requester-optimal outcome with error at most $\delta\leq0.1$ in both markets needs expected pulls of that edge at least $\log(1/\delta)/(64\eta^2)$ in either market. If that edge has fixed reporting delay $d_c$ and all others are immediate, expected reward rounds before identification are at least $(1-2\delta)(d_c+1)$.
\end{proposition}
The statistical claim follows by change of measure; the delay claim by indistinguishability before any pivotal observation is available (Appendix~\ref{app:separations}). This is an identification lower bound. It is not a matching regret lower bound or a proof that CGR's logarithmic number of batch waits is necessary.

\paragraph{Scope of the notification assumption.}
Without a notification or a valid drift envelope, old confidence intervals can remain narrow immediately after an arbitrary change and exclude the new truth. Two markets can have identical observed histories but different current DA outcomes. A certificate built from that history cannot certify both. An unannounced-change extension must pay for monitoring and detection or impose additional drift conditions. CGR does neither; its theorem is explicitly about announced repair.

%% file: sections/trace_proofs.tex
\section{Proofs for certificate-guided sampling and repair}\label{app:newproofs}
\subsection{Adaptive chronological concentration}
Fix an edge and one evidence lifetime. Its mean $\mu_e$ is constant. Let $X_s$ indicate that a reward from this edge is issued for the estimator at round $s$ during this lifetime. This indicator is predictable with respect to the full latent past; CGR's decisions use a subset of that information. Even though the choice of the next batch is adaptive, all its pulls are chosen before their corresponding noises are generated. Hence
\begin{equation}
 Z_t(\lambda)=\exp\left(\lambda\sum_{s\leq t}X_s(Y_{e,s}-\mu_e)-\frac{\lambda^2\sigma^2}{2}\sum_{s\leq t}X_s\right)
\end{equation}
is a nonnegative supermartingale starting at 1. Stop it at the $n$th issued sample, the end of the lifetime, or $T$, whichever comes first. On the event that $n$ samples are issued and their centered sum exceeds $nx$, Markov's inequality and the stopping argument give
\begin{equation}
 \Prb\left(n\text{ samples exist and }\frac1n\sum_{r=1}^n(Y_{e,s_r}-\mu_e)>x\right)
 \leq\exp\left(-\frac{nx^2}{2\sigma^2}\right).
\end{equation}
Here $s_r$ is the action time of the $r$th issued sample. The lower tail is identical. If $\sigma=0$, the noise is zero almost surely. At most $T$ lifetimes and $T$ sample counts occur per edge. Substituting $x=r(n)$ and union bounding gives failure probability at most $\delta$.

CGR only updates after a whole selected batch is available, and all earlier batches in the lifetime were already complete. Its estimator therefore uses an initial prefix of the chronologically issued samples, at a possibly random count. The simultaneous count bound covers this randomness, including reward-dependent completion times. It does not assert concentration for arbitrary subsets of early arrivals. Intersecting sound intervals preserves the mean; the fallback for empty intersections is never used on this event. Unchanged edges keep the same mean across other edges' epoch boundaries, so their lifetime confidence remains valid.

\subsection{The widest-candidate argument}
At an unresolved query let $w$ be the unique true winner and $c$ the sampled real candidate. Let $s_a=u_a-\ell_a$ be interval width. Because $\mu_w\geq\mu_a\geq\ell_a$ for every $a$, and $u_w\geq\mu_w$, the winner belongs to $\mathcal P$. All statements below are on the confidence event.

If $c\ne w$, then $u_c\geq\max_a\ell_a\geq\ell_w$. Thus
\begin{equation}
 \mu_w-\mu_c\leq(\mu_w-\ell_w)+(u_c-\mu_c)\leq s_w+s_c\leq2s_c.
\end{equation}
The last inequality follows because $c$ has largest width among real possible winners, and the outside option has width zero.

If $c=w$, then $w$ is real. If some $a\ne w$ has $\ell_a>\ell_w$, then $\mu_w-\mu_a\leq\mu_w-\ell_w\leq s_w$. Otherwise $\ell_w=\max_a\ell_a$. Since the query is unresolved, some $a\ne w$ has $u_a\geq\ell_w$. This alternative belongs to $\mathcal P$, and
\begin{equation}
 \mu_w-\mu_a\leq s_w+s_a\leq2s_w.
\end{equation}
In all cases a comparison involving the selected edge has gap at most twice its width. That query is on the true trace by induction using Theorem~\ref{thm:trace}; hence $\Gamma_{e,h}\leq2s_c$.

For $n\geq n^\star_{e,h}$, $s_c\leq2r(n)\leq\Gamma_{e,h}/\sqrt{8}<\Gamma_{e,h}/2$, a contradiction. With $\sigma=0$ the same conclusion holds after one sample. An edge not involved in any nontrivial true-trace comparison cannot be selected.

\subsection{Doubling, lifetime reuse, and dwell time}
A batch starts only below its epoch threshold. Its completed count is at most twice that threshold, including the initial target 1. Counts increase through $1,2,4,\ldots$ within a lifetime, so even if another epoch needs a smaller margin, the final count is at most twice the maximum threshold over that lifetime. At most $1+\lceil\log_2(2n^\star_{e,g})\rceil$ batches are issued. The same cap holds for a final horizon-truncated batch. Issued samples and completed counts coincide between batches; there is no discarded-and-repeated batch in the completed nonfinal epochs covered here.

If a batch's last pull occurs in round $u$, every reward in it has arrived by the end of $u+d_e$. The next action after pulling is $u+1$, so at most $d_e$ further reward rounds are spent waiting. This gives Eq.~\eqref{eq:selectivebudgets}. Without retained evidence, the identical argument bounds each epoch's work by $F_h$; retained counts can only remove some of those batches, since an edge above its threshold cannot be sampled. This proves the sufficient duration statement. If a notification nevertheless interrupts an unfinished repair, define a fallback that stops learning and remains unmatched until $T$. Under the stated dwell condition this cannot happen on the confidence event. The selective lifetime theorem is not claimed for a protocol that keeps processing overlapping repairs. The final epoch may simply terminate unfinished at $T$.

A more explicit warm-start bound is available. If edge $e$ starts epoch $h$ with $n^0_{e,h}\geq n^\star_{e,h}$, it needs no new samples in that epoch. Otherwise it issues at most $2n^\star_{e,h}-n^0_{e,h}$ additional samples. Sound prior intervals may exclude an edge even earlier. This is why an edge with large $d_e$ can have no waiting charge after an update.

\subsection{Full stability and expectation accounting}
Let $M$ differ from a stable matching $M^\star$ on requester set $A$ of size $q$. Provider partners differ on a set $J$ of size at most $2q$: each changed provider is an old or a new partner of a requester in $A$. A blocking pair with neither endpoint in $A\cup J$ would also block $M^\star$. Thus at most $(K+2N)q$ blocking pairs arise. Only requesters in $A$ can acquire a negative-mean partner, giving at most $q$ additional individual-rationality violations. This proves the stability bound in Theorem~\ref{thm:local-repair}.

The public feasible graph already excludes provider-unacceptable edges. During an exploratory pull, at most the exploratory requester and the old occupant of its target provider change assignments relative to $M^{\rm old}$; all remaining old pairs remain feasible. Therefore the regret modification is at most 4. During waiting, the allocation is exactly $M^{\rm old}$. On the uniform confidence event, every previous completed matching is the true requester-optimal outcome of its epoch. These observations justify the induction over announced updates and Eqs.~\eqref{eq:localregret}--\eqref{eq:localstability}.

Write $\mathcal G$ for the confidence event. On $\mathcal G^c$, regret is bounded by $2NT$ and violations by $N(K+1)T$. Taking expectations of the pathwise inequalities on $\mathcal G$, upper bounding their nonnegative right-hand sides by unconditional expectations, and charging $\Prb(\mathcal G^c)\leq\delta$ gives the stated expectation guarantees. For Eq.~\eqref{eq:detregret}, replace $P,W$ by their deterministic caps on $\mathcal G$ and then apply the failure charge. No random sampled-edge set is silently treated as deterministic.

\section{Separation proofs and lower-bound details}\label{app:separations}
\subsection{A connected rejection cascade}
Use requesters and providers $0,\ldots,N-1$, with edges $(i,i)$ and $(i,i+1\bmod N)$. Provider $j$ ranks $j-1\bmod N$ above $j$, both acceptable. Initially each requester strictly prefers its own provider, so the identity matching is returned immediately by DA. After requester 0 swaps its two utilities, it first proposes to provider 1. Requester 1's proposal to provider 1 is rejected in favor of 0, so 1 proposes to 2. This repeats until requester $N-1$ proposes to provider 0. Every requester now has provider $i+1\bmod N$, and no one remains unmatched. Because these are actual DA runs, both stated outcomes are requester-optimal. The graph is connected and only requester 0 changed preferences, yet all partners change.

Suppose retained intervals on every unchanged edge have radius at most $0.05$ around the true mean. They certify both $0.8>0.3$ and $0.3>0$. Once requester 0's new choice is learned, every subsequent query in the new trace is already certified from those unchanged intervals: requester $i>0$ first chooses $i$, and after rejection chooses $i+1$ over the outside option. Only requester 0 can require fresh samples. This contrasts statistical evidence reuse with the number of changed assignments, without asserting that arbitrary sparse updates have local allocation effects.

\subsection{Testing a pivotal preference}
In Eq.~\eqref{eq:pivotal}, DA gives $(1a,2b)$ in the plus market and $(1b,2a)$ in the minus market. All observations except rewards from edge $(1,a)$ have the same conditional distribution in both worlds; provider acceptances and fixed delays contain no extra information. Let $N_c$ be its number of pulls by the algorithm's output time. Standard likelihood-ratio factorization for adaptive experiments gives
\begin{equation}
 \operatorname{KL}(\mathbb P_+,\mathbb P_-)\leq\E_+[N_c]\,
 \operatorname{kl}(1/2+\eta,1/2-\eta).
\end{equation}
For delayed data, this can be applied to an augmented record revealing all issued rewards at the stopping time: the actual observed record is a function of that record, so data processing gives the displayed upper bound as well. Stopping-time arguments follow by truncation if the expected count is finite; otherwise the lower bound is immediate.

Let the event be outputting the plus-market optimal matching. Correctness in both worlds and data processing imply transcript KL at least $\operatorname{kl}(1-\delta,\delta)\geq\tfrac12\log(1/\delta)$ for $\delta\leq0.1$. For $p,q\in[1/4,3/4]$, $\operatorname{kl}(p,q)\leq(p-q)^2/(q(1-q))$, so the pivotal KL is at most $32\eta^2$. Combining the two inequalities proves the sample lower bound. Interchanging the markets proves the other direction.

If pivotal feedback is delayed by $d_c$, it cannot be observed before $d_c+1$ reward rounds have completed, even if its first pull occurs in round 1. The distribution of an output made before that time is the same in both worlds. Since their correct outputs differ, the sum of their early-output error probabilities is at least the common probability of early output. Each total error is at most $\delta$, so early output has probability at most $2\delta$. Therefore expected reward rounds before output are at least $(1-2\delta)(d_c+1)$ in either world. This argument concerns correct identification, not positive-part regret: an unstable or differently stable allocation can hedge requester losses.

\subsection{A certificate need not learn every losing comparison}
In Proposition~\ref{prop:gap-separation}, the true trace contains one query, comparing $0.8$ to $0.2$, $0.2+\eta$, and zero. Thus finite margins are $0.6$, $0.6-\eta$, or their minimum. A confidence radius small enough for these comparisons need not separate $0.2$ from $0.2+\eta$. At $\eta=0$, the trace theorem still applies; a procedure insisting on a strictly certified order of those losing arms has no finite certificate. This is not an impossibility result for algorithms that explicitly model indifference.

%% file: sections/stability_reference.tex
\section{Baseline stability reduction and its limitation}
\label{sec:certificate}
Let $g(M,t)$ be the maximum of zero, the gains of provider-willing blocking pairs, and requester individual-rationality shortfalls. This measures deviation incentives, not optimal stable regret.
Let $z_{i,j}\in[-1,1]$ be scores, with $z_{i,0}=0$, and let $\widehat M=\DA(z)$. Requesters propose only to providers ranked strictly above the outside option; rejected providers are deleted. A provider compares incoming proposals with its current tentative partner, retaining its favorite. This last comparison is essential for DA stability.

\begin{theorem}[Stability certificate and exact recovery]\label{thm:certificate}
Fix a round $t$ and suppose $|z_{i,j}-\mu_{i,j,t}|\leq e_{i,j}$ for all feasible pairs. Set $e_{i,0}=0$ and $e=\max_{(i,j)\in\cE}e_{i,j}$. Then
\begin{equation}
 g(\widehat M,t)\leq2e.\label{eq:certificate}
\end{equation}
More precisely, each provider-willing pair's positive deviation gain is at most $e_{i,j}+e_{i,\widehat M(i)}$, and each requester's individual-rationality shortfall is at most $e_{i,\widehat M(i)}$. Define
\begin{equation}
 \Delta_t=\min_i\min_{\substack{j,j'\in F_i\cup\{0\}\\j\ne j'}}|\mu_{i,j,t}-\mu_{i,j',t}|,\label{eq:gap}
\end{equation}
where the minimum over an empty comparison set is $+\infty$. If $2e<\Delta_t$, then $\widehat M=M_t^\star$.
\end{theorem}
\begin{proof}
DA is stable under scores (Appendix~\ref{app:da}). If a provider prefers $i$ to its assigned partner, score stability implies $z_{i,j}\leq z_{i,\widehat M(i)}$. Adding the two score errors bounds the true gain. Score individual rationality implies $z_{i,\widehat M(i)}\geq0$, proving the shortfall bound. If $2e<\Delta_t$, no pair of alternatives, including the outside option, can reverse order. DA therefore receives identical strict preference lists and returns the same matching.
\end{proof}

\begin{proposition}[Stable-optimal regret is not Lipschitz in score error]\label{prop:discontinuity}
There are two-requester markets with arbitrarily small uniform score error $e>0$ for which the score-DA outcome has positive-part regret at least $4/5$ in one round, although its maximum deviation incentive is at most $e$.
\end{proposition}
\begin{proof}
Let providers be $a,b$, with $1\succ_a2$ and $2\succ_b1$, both above the outside option. For $0<\eta<1/4$, let
\begin{equation}
 \mu=\begin{pmatrix}1/2-\eta&1/2\\9/10&1/10\end{pmatrix},\qquad
 z=\begin{pmatrix}1/2+\eta&1/2\\9/10&1/10\end{pmatrix}.\label{eq:counterexample}
\end{equation}
True DA returns $(1b,2a)$; score DA returns $(1a,2b)$. The uniform error is $e=2\eta$, whereas requester 2 loses $4/5$. The latter matching has no blocking pair: requester 1 prefers $b$, but $b$ prefers its current partner 2; requester 2 prefers $a$, but $a$ prefers 1. It is actually exactly stable under $\mu$, so $g=0$.
\end{proof}
The example is stronger than approximate instability: even exact stability does not imply proximity to the requester-optimal outcome. It disproves a uniform bound $R_t\leq C\|z-\mu_t\|_\infty$ for the DA map with fixed $C$. It is not a minimax lower bound on every learning algorithm.

\section{Deferred acceptance and implementation details}\label{app:da}
Fix strict provider priorities and requester score lists, including the outside option. Initially everyone is unmatched. An unmatched requester with a remaining provider above the outside option proposes to its highest remaining provider. The provider keeps its preferred requester from the new proposals and its incumbent tentative match, rejecting all others. Rejected requesters delete that provider; tentative partners do not submit new proposals. Termination occurs when no unmatched requester has an acceptable untried provider. Scores tied at zero are placed below the outside option. Other score ties use a fixed provider-ID order.

Each feasible pair is proposed at most once. Thus at most $|\cE|$ messages are needed. Provider tentative partners only improve in the provider order. Suppose a pair blocks the final matching. If the requester never proposed, it ended with a more preferred match (or the provider was unacceptable), contradicting blocking. If it proposed and was rejected, the provider ended with someone it prefers to that requester. Thus no score-blocking pair exists. Requesters never propose below the outside option, establishing score individual rationality.

For completeness, requester optimality follows from the rejection argument. Suppose DA rejects an edge belonging to some stable matching, and consider the first such rejection in the execution. Provider $j$ rejects requester $i$ for $i'\succ_j i$. In that stable matching $i'$ must prefer its partner to $j$, or $(i',j)$ would block. But $i'$ proposed to $j$, so any preferred partner must have rejected $i'$ earlier. That would be an earlier rejection of an edge belonging to a stable matching, a contradiction. No edge in any stable matching is rejected. Therefore every requester obtains its best achievable stable partner. The same argument applies after the specified tie refinement.

%% file: sections/new_experiment_details.tex
\section{New experiment details and complete summaries}\label{app:newexperiments}
\paragraph{Reproducible implementation.}
The paired runner uses Python's standard library and four local CPU worker processes. The previously used remote server could not be reached during this revision; no new run is reported as an A40 run. The 240 tasks consist of 120 irrelevant-comparison tasks (four gaps, four methods), 90 dynamic tasks (three delay placements, six methods), and 30 signed-Jester tasks (three methods). These give 1,110 method trajectories and 2,730 epoch runs. All 2,160 dynamic epoch outputs are correct and complete, including 1,620 nonfinal repairs. With $\rho=.1$, the common sufficient dwell bound is 1,593,800 rounds, below the fixed 5,000,000-round epoch. This calculation is separate from observing completion in a finite set of seeds.

Per-group logs contain selected edges, intended and issued sample counts, exploration rounds, and actual waiting. No partial final group is used for an interval update. After an incomplete final repair, the old matching continues to the horizon and its loss is included. Regret/violation pathwise bounds and feasibility were checked in all dynamic epochs. Twenty-seven implementation tests include interval-box endpoints and interior completions, exact ties, signed outside options, three-requester oracle checks, parallel conflicts, and delayed truncation. These finite checks support the implementation but do not prove the theorems. Configurations and data/source hashes accompany the raw outputs.

\paragraph{Prefix comparison.}
We adapt the stopping condition of Algorithm 3 in \citet{athanasopoulos2025pcos}. Compute DA from interval midpoints, and require every alternative scored at least as high as the resulting partner to have a disjoint interval from every other alternative. The outside option is an alternative with interval $[0,0]$. The ordered prefix down to the proposed final partner is then fixed for every completion; running DA cannot depart from these prefixes, proving soundness. Our matched-width scheduling and common horizon-dependent radius replace the source algorithm's elimination schedule and radius. Thus Prefix-P isolates a literature-based stopping condition under common experimental machinery. It is not a claim to have implemented the best published PCOS or pervasive-matching algorithm.

\paragraph{Jester preprocessing.}
We use Jester's first file and its universal-query items 5, 7, 8, 13, 15, 16, 17, 18, 19, 20. Of 24,983 users, 24,953 have all these ratings observed. Shuffle eligible rows with seed 202709. Seed $s$ uses rows $10s,\ldots,10s+2$ and the first five query items, yielding disjoint users across seeds. Provider priorities are independent random permutations generated with seed $861+s$. Every rating is retained unchanged before scaling by $1/10$. Signed rewards take values $\pm1$, with mean equal to that scaled rating, and use $\sigma=1$. This benchmark uses a different generator from the controlled affine suite and is not pooled with it.

\paragraph{Actual data-training jobs.}
We use the 1,797 public optical-digit observations distributed by \texttt{sklearn.datasets.load\_digits} (64 pixel features scaled by $1/16$). A stratified 70/30 split with seed 202709 separates candidate training data from held-out evaluation. Buyer tasks distinguish digits $(0,1)$, $(3,8)$, and $(4,9)$. Each buyer owns a fixed balanced anchor of ten training observations, five per class. Four task-specific provider offers contain 20, 40, 80, or 120 additional training observations, chosen without using held-out performance. Offers may overlap; capacity and priority constraints are not inferred from this dataset.

For each buyer--offer pair, 30 independent seeded bootstrap resamples of that fixed offer are joined to the anchor. Every trial trains a new binary logistic model from zero weights: 150 full-batch gradient steps, learning rate $0.2$, and weight penalty $0.001$. The anchor-only model uses the same procedure. Utility is the difference in held-out accuracy, so negative gains are possible. All 360 trials run actual training and validation; the reported durations cover both computations, exclude import/data-loading time, and do not include queueing. CPU timing is hardware/load dependent and is not used to infer a deterministic delay bound or comparative algorithm speed. There is no hyperparameter selection on these outcomes. The fixed held-out set supports within-task resampling comparisons, not independent population-level generalization claims. Source-row indices, splits, offers, baseline scores, jobs, NumPy/scikit-learn versions, and source/data hashes are saved.

\input{results/outcome/summary_table}

%% file: results/outcome/summary_table.tex
\begin{longtable}{lllrrr}
\caption{All new matching configurations, 30 seeds each. Costs are mean learning rounds, capped at $10^5$ in stationary cases; dynamic costs sum four repairs. Total regret includes commitment and incomplete runs.}\label{tab:alloutcome}\\
\toprule Family & Setting & Method & Rounds & Regret & Complete \\
\midrule\endfirsthead
\toprule Family & Setting & Method & Rounds & Regret & Complete \\ \midrule\endhead
irrelevant & 0 & trace-P & 100000 & 52429 & 0/30 \\
irrelevant & 0 & outcome-P & 230 & 102 & 30/30 \\
irrelevant & 0 & outcome-S & 644 & 660 & 30/30 \\
irrelevant & 0 & prefix-P & 100000 & 52429 & 0/30 \\
irrelevant & 0.001 & trace-P & 100000 & 52429 & 0/30 \\
irrelevant & 0.001 & outcome-P & 230 & 102 & 30/30 \\
irrelevant & 0.001 & outcome-S & 644 & 660 & 30/30 \\
irrelevant & 0.001 & prefix-P & 100000 & 52429 & 0/30 \\
irrelevant & 0.01 & trace-P & 100000 & 52429 & 0/30 \\
irrelevant & 0.01 & outcome-P & 230 & 102 & 30/30 \\
irrelevant & 0.01 & outcome-S & 638 & 655 & 30/30 \\
irrelevant & 0.01 & prefix-P & 100000 & 52429 & 0/30 \\
irrelevant & 0.1 & trace-P & 14473 & 6663 & 30/30 \\
irrelevant & 0.1 & outcome-P & 230 & 102 & 30/30 \\
irrelevant & 0.1 & outcome-S & 576 & 599 & 30/30 \\
irrelevant & 0.1 & prefix-P & 14473 & 6663 & 30/30 \\
dynamic & none & trace-retain & 7234 & 10653 & 30/30 \\
dynamic & none & trace-reset & 7345 & 10008 & 30/30 \\
dynamic & none & outcome-retain & 7234 & 10653 & 30/30 \\
dynamic & none & outcome-reset & 6321 & 9437 & 30/30 \\
dynamic & none & prefix-retain & 7234 & 10653 & 30/30 \\
dynamic & none & prefix-reset & 7345 & 10008 & 30/30 \\
dynamic & changed & trace-retain & 12665 & 22812 & 30/30 \\
dynamic & changed & trace-reset & 12689 & 22007 & 30/30 \\
dynamic & changed & outcome-retain & 12665 & 22812 & 30/30 \\
dynamic & changed & outcome-reset & 11572 & 21389 & 30/30 \\
dynamic & changed & prefix-retain & 12665 & 22812 & 30/30 \\
dynamic & changed & prefix-reset & 12689 & 22007 & 30/30 \\
dynamic & unchanged & trace-retain & 9045 & 18043 & 30/30 \\
dynamic & unchanged & trace-reset & 12744 & 22341 & 30/30 \\
dynamic & unchanged & outcome-retain & 9045 & 18043 & 30/30 \\
dynamic & unchanged & outcome-reset & 11661 & 21740 & 30/30 \\
dynamic & unchanged & prefix-retain & 9045 & 18043 & 30/30 \\
dynamic & unchanged & prefix-reset & 12744 & 22341 & 30/30 \\
signed-jester & signed & trace-P & 95324 & 122907 & 2/30 \\
signed-jester & signed & outcome-P & 91973 & 117466 & 4/30 \\
signed-jester & signed & prefix-P & 95324 & 122907 & 2/30 \\
\bottomrule
\end{longtable}